\documentclass[runningheads]{llncs}
\usepackage[T1]{fontenc}
\usepackage{url}
\usepackage{amsmath}
\usepackage{amssymb}
\usepackage{mathtools}
\usepackage{times}
\usepackage{latexsym}
\usepackage{url}
\usepackage{times}
\usepackage{latexsym}
\usepackage{subfigure}
\usepackage{graphicx}
\usepackage{mathrsfs}
\usepackage{bbm}
\usepackage{algorithm2e}
\usepackage{color}
\usepackage[svgnames]{xcolor}
\usepackage{multirow}
\usepackage{booktabs}
\usepackage{arydshln}
\usepackage{enumitem}
\usepackage{tabularx}
\usepackage{makecell}
\usepackage{fancyhdr}
\usepackage{wrapfig}
\usepackage[misc]{ifsym}
\usepackage[numbers]{natbib}
\newcommand{\corr}{(\Letter)}
\usepackage{mwe}
\newtheorem{assumption}{Assumption}
\begin{document}

\title{Generalizing Reward Modeling for Out-of-Distribution Preference Learning}

\titlerunning{Underwater Basket Weaving Under Extreme Pressure}

\author{Chen Jia \orcidID{0000-0002-8666-9930} \corr}

\authorrunning{Jia}

\institute{SI-TECH Information Technology \email{jiachenwestlake@gmail.com}}

\maketitle              

\begin{abstract}
Preference learning (PL) with large language models (LLMs) aims to align the LLMs' generations with human preferences. Previous work on reinforcement learning from human feedback (RLHF) has demonstrated promising results in in-distribution PL. However, due to the difficulty of obtaining human feedback, discretely training reward models (RMs) for every encountered distribution is challenging. Thus, out-of-distribution (OOD) PL is useful for enhancing LLMs' generalization ability with limited preference feedback. This work addresses OOD PL by optimizing a general RM through a meta-learning approach. A bilevel optimization algorithm is utilized during meta-training to learn an RM that guides policy learning to align with human preferences across various distributions. When encountering a test distribution, the meta-test procedure optimizes regularized policy using the learned RM for PL. We theoretically demonstrate the convergence rate of the bilevel optimization algorithm under reasonable assumptions. Additionally, we conduct experiments on two text generation tasks across 22 held-out data distributions and outperform various strong baselines across various evaluation metrics.
\end{abstract}

\section{Introduction}\label{sec1}
Aligning large language models (LLMs) with human preferences through reinforcement learning has been demonstrated as a practical approach to align pretrained LLMs along human values. As shown by recent research on LLMs \cite{christiano2017deep,ziegler2019fine,stiennon2020learning,ouyang2022training}, RLHF initially trains a reward model (RM) to capture human preferences from a pairwise preference dataset. It then aligns the LLMs with the learned RM through regularized policy optimization, aiming to learn a language policy that better reflects human values. Most previous work on RLHF focuses on in-distribution (ID) preference learning (PL) \cite{casper2023open}, i.e., using the ID preference data for reward learning \cite{christiano2017deep} and then performing policy optimization using PPO (Fig. \ref{fig:intro} (b)). Besides, direct preference optimization (DPO) \cite{rafailov2024direct} and its variants \cite{azar2024general,ethayarajh2024kto,park2024disentangling} directly optimize a language policy using maximum likelihood estimation on the ID preference data (Fig. \ref{fig:intro} (c)). 

However, defining precise rewards for various real-world tasks is non-trivial \cite{mckinney2022fragility}, and obtaining high-quality feedback that accurately represents human preferences is challenging \cite{bai2022training}.
Therefore, we focus on enhancing the out-of-distribution (OOD) generalization ability for PL. Fig. \ref{fig:intro} (a) illustrates an OOD scenario where only preference data from training distributions (\fcolorbox{black}{pink}{$y$},\fcolorbox{black}{pink}{\footnotesize$y'$},..., \fcolorbox{black}{green}{$y$},\fcolorbox{black}{green}{\footnotesize$y'$}) are available, as well as fine-tuning data from training-/test-distributions ($\mathcal{D}^{{tr}_1}_{xy}$,...,$\mathcal{D}^{{tr}_N}_{xy}$, $\mathcal{D}^{{te}}_{xy}$). Two major challenges arise when these ID PL methods encounter the distribution shift: (i) the generations from the policy may substantially deviate from the human preferences of the target distribution, posing an OOD challenge for PL. (ii) the distribution shifts result in policy drift during the policy optimization process with Kullback-Leibler divergence \cite{ramamurthy2022reinforcement,rafailov2024direct}: the policy tends to move towards the reference policy largely followed in the training distribution, exacerbating the model's deviation from the test distribution in PL. 
\begin{wrapfigure}[16]{r}{0.4\textwidth}
	\centering
	\includegraphics[width=1.\linewidth]{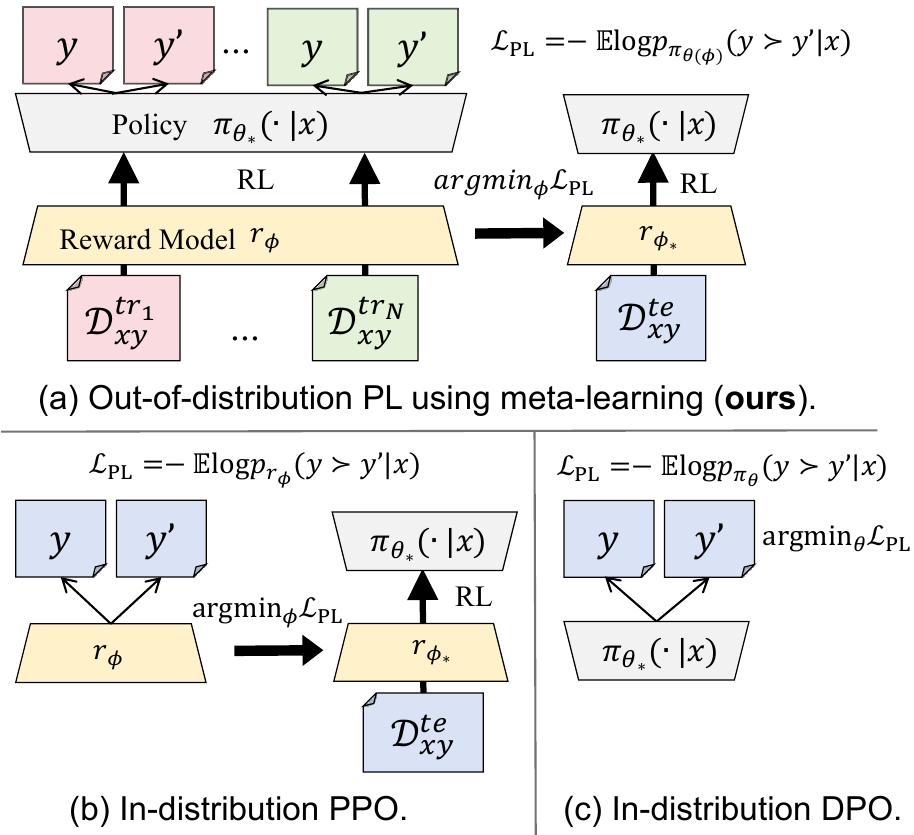}
	\caption{Comparison between OOD PL and existing ID PL, i.e., PPO and DPO.} \label{fig:intro}
\end{wrapfigure}

An initial approach to addressing the OOD challenge is to directly enhance the OOD generalization ability of the RM using preference data from multiple distributions. This specific approach may involve training an RM using multiple distributions or integrating distribution-specific RMs into a unified model, similar to previous work on RM ensembling \cite{rame2024warm}. However, such a method also has limitations from the distribution shift between the preference data and fine-tuning data. To better tackle the OOD challenge, we employ a meta-learning approach to learn a reward function capable of guiding policy optimization for OOD preference learning (Fig. \ref{fig:intro} (a)). For each episode, the policy is optimized with RL on fine-tuning data (meta-training on support set) and the RM is optimized with PL objectives on the preference data (meta-test on query set). In particular, we incorporate a regularization term in policy optimization to penalize the Kullback-Leibler divergence between the policy-induced distribution and test distribution, thereby mitigating the challenge of policy drift. We summarize the contributions as follows.
\begin{itemize}[leftmargin=1em]
	\item {\bf Improving OOD generalization for preference learning.} We focus on aligning large language models with human preference on the OOD data, and propose a novel meta-learning-based approach to tackle the OOD preference learning problem.
	\item {\bf Gradient-based bilevel optimization algorithm.} To optimize the bilevel optimization objectives of meta-learning, we propose a gradient-based algorithm, and we establish an upper bound for its convergence rate w.r.t. the learning rates and controlling factor for reward modeling.
	\item {\bf State-of-the-art performance.} We conduct experiments on both controlled sentiment generation and knowledge answer generation, using multiple metrics for evaluation. The results outperform a range of strong baselines and achieve the best results across four sentiment generation distributions and 18 answer generation distributions. The code is released at \url{https://github.com/jiachenwestlake/OODPL}.
\end{itemize}

\section{Preliminaries}
\noindent\textbf{Notations.}
We relay on an off-policy sequence-to-sequence decision making framework of language generation \cite{baheti2023improving,azar2024general,rafailov2024direct}: given a task-specific prompt $x \sim \mathcal{D}_{x}$, the $\theta$-parameterized language policy $\pi_{\theta}$ computes the probability distribution of response $y \sim \pi_{\theta}(\cdot | x)$, together with a task-specific sequence-level reward $r_{\phi}(x,y) \in \mathbb{R}$ by the $\phi$-parameterized RM $r_{\phi}$.
Preference learning (PL) \cite{christiano2017deep,ouyang2022training,rafailov2024direct} further utilizes answer comparisons corresponding to a prompt, $\{y, y', x: y\succ y'| x\}$, drawn from a preference distribution $\mathcal{D}_{\rm PL}$ where $y \succ y' | x$ indicates that $y$ is the preferred answer over dispreferred answer $y'$ given the prompt $x$. 
\begin{wraptable}[10]{r}{0.5\textwidth}
	\vspace{-0.8cm}
	\centering
	\caption{Comparison of training data distributions. $\mathcal{D}_{\rm PL}^{tr}$ and $\mathcal{D}_{\rm PL}^{te}$ represent training and testing preference distributions, respectively, $\mathcal{D}^{tr_i}_{xy}$ denote the RL (or SFT) data distribution.}
	\scalebox{.75}{
		\begin{tabular}{lll}
			\toprule
			{\bf Method}& PL & RL (SFT) \\
			\midrule
			SFT (ID) & - &$\mathcal{D}^{te}_{xy}$ \\
			$^{\llcorner}$PPO (ID) & {\small $\mathcal{D}^{te}_{\rm PL}$}  & $\mathcal{D}^{te}_{xy}$  \\
			$^{\llcorner}$xDPO (ID)& {\small$\mathcal{D}^{te}_{\rm PL}$} & -   \\
			$^{\llcorner}${\bf ours} (OOD) & \colorbox{pink}{\small$\{\mathcal{D}_{\rm PL}^{tr_1},...,\mathcal{D}_{\rm PL}^{tr_N}\}$} & $\mathcal{D}^{te}_{xy}$, \colorbox{pink}{\small$\{\mathcal{D}^{tr_1}_{xy},..., \mathcal{D}^{tr_N}_{xy}\}$} \\
			\bottomrule
		\end{tabular}
	}
	\label{tab:problem}
\end{wraptable}
\noindent\textbf{Supervised fine-tuning (SFT).} 
SFT involves employing a supervised learning objective, specifically maximum likelihood estimation (MLE), to fine-tune a pretrained LLM on a task-specific dataset, denoted as $\mathcal{D}_{xy}$. The training objective is to obtain a language policy $\pi_{\theta}$ that satisfies
\begin{align}
	\min_{\theta} - \mathbb{E}_{(x,y)\sim \mathcal{D}_{xy}} \left[ \log \pi_{\theta}(y|x) \right]
\end{align} 

\noindent\textbf{RLHF with reward modeling.} 
The standard RLHF framework \cite{ziegler2019fine,ouyang2022training} comprises two main phases: (i) learning an RM for a specific task, and (ii) fine-tuning the policy with the learned RM. During reward learning, given a preference distribution $\mathcal{D}_{\rm PL}$, an RM $r_{\phi}$ is trained to score preference of $y$ to be the answer of $x$, using a maximum likelihood estimation (MLE) loss function:
\begin{align}
	\mathcal{L}_{R} = - \mathbb{E}_{(x,y \succ y') \sim \mathcal{D}_{\rm PL}} \left[ \log \sigma\left(  r_{\phi}(x,y) - r_{\phi}(x,y') \right) \right],
\end{align}
where $\sigma$ denotes the logistic function.
In the RL fine-tuning phase, given a prompt distribution $\mathcal{D}_x$, the optimization problem can be represented as follows,
\begin{align}\label{eq:rlhfobj}
	\max_{{\theta}} \mathbb{E}_{x \sim \mathcal{D}_x} \left[ \mathbb{E}_{y \sim \pi_{\theta}(\cdot|x)} \left[ r_{\phi}(x,y) \right] - \beta \mathrm{D}_{\rm KL} \left( \pi_{\theta}(\cdot|x) || \pi_{\rm ref}(\cdot|x) \right)\right],
\end{align}
where $\mathrm{D}_{\rm KL}(\cdot||\cdot)$ represents the Kullback-Leibler divergence, $\pi_{\rm ref}(\cdot|x)$ denotes a reference distribution, and $\beta$ is the balancing coefficient. Previous work \cite{ramamurthy2022reinforcement,rafailov2024direct} uses an SFT model as the $\pi_{\rm ref}$ to impose constraints on the similarity between the predictions of the RLHF model and those of the initial SFT model. We define $\pi_{\rm ref}$ as the target distribution to address the policy drift challenge, as detailed in \S \ref{sec:method}.

\noindent\textbf{Problem Setting.}
We compare three PL methods in Tab. \ref{tab:problem}. To ensure a fair comparison, we consider a transfer learning scenario where all methods employ prompting from the test distribution $\mathcal{D}^{te}_{y|x}$ during testing. In contrast to supervised fine-tuning (SFT), preference learning (PL) utilizes additional preference data for training, as demonstrated in previous works \cite{christiano2017deep,ouyang2022training,zhang2023wisdom,liu2023aligning,rafailov2024direct}. In the ID PL setting, both RLHF-PPO \cite{christiano2017deep,ouyang2022training} and xDPO \cite{rafailov2024direct,azar2024general} relay on the ID preference data $\mathcal{D}^{te}_{\rm PL}$ for reward learning or policy optimization. This study concentrates on an OOD scenario, where we utilize preference data $\{\mathcal{D}_{\rm PL}^{tr_1},...,\mathcal{D}_{\rm PL}^{tr_N}\}$ and RL fine-tuning data $\mathcal{D}^{te}_{xy}$, $\{\mathcal{D}^{tr_1}_{xy},..., \mathcal{D}^{tr_N}_{xy}\}$ from the training distributions for preference learning.

\section{Meta-Learning for Generalizing Reward Modeling} \label{sec:method}
\begin{table*}[h]
	\centering
	\caption{Linkes among different fields of bilevel optimization in this work.}
	\scalebox{0.7}{
		\begin{tabular}{lp{6cm}p{6.5cm}}
			\toprule
			{Bilevel Programming}&{ID Preference Learning} &{OOD Preference Learning (Meta-Learning)} \\
			\midrule
			outer variables & RM $\phi$ & generalized RM $\phi$ \\
			outer objective & PL objective $\mathcal{L}_{\rm PL}(\phi, \theta_*, \mathcal{D}_{\rm PL})$ (Eq. (\ref{eq:outer}))& PL objective ${\bf L}(\phi)$ (Eq. (\ref{eq:metaobjouter}))  \\
			inner variables & policy $\theta$ & policies of each distribution $\{\theta^T\}_{T\in \mathcal{D}_{\mathcal{T}}}$\\
			inner objective & FT objective $\mathcal{L}_{\rm FT}(\phi, \theta, \mathcal{D}_{xy}) $ (Eq. (\ref{eq:innerobj})) & FT objectives for each distribution $\{\mathcal{L}_{\rm FT}(\phi, \theta, \mathcal{D}^T_{xy})\}_{T\in \mathcal{D}_{\mathcal{T}}} $ (Eq. (\ref{eq:metaobjinner})) \\
			\bottomrule
		\end{tabular}
	}
	\label{tab:bilevel}
	\vspace{-0.5cm}
\end{table*}

Our meta-learning approach for OOD generalization is built upon bilevel programming, \cite{finn2017model,li2018learning,ji2021bilevel}. In this section, we first introduce a bilevel optimization framework for ID preference learning (PL) in \S \ref{sec:bilevelrm}. Based on this, we introduce a meta-learning approach for OOD PL in \S \ref{sec:metalearn}. Tab. \ref{tab:bilevel} outlines the connections among bilevel programming, ID preference learning, and meta-learning for OOD PL. Additionally, we introduce a gradient-based algorithm for optimizing the meta-learning objective in \S \ref{sec:algorithm} and analyze the theoretical convergence rate in \S \ref{sec:convergence}.

\subsection{Preference Learning via Bilevel Programming} \label{sec:bilevelrm}
\noindent\textbf{Bilevel programming.}
We formularize preference learning as a bilevel optimization problem (see e.g., \cite{colson2007overview,sinha2017review}):
\begin{align}
	\vspace{-0.5cm}
	&\min_{\phi} \mathcal{L}_{\rm PL}({\phi}, {\theta_*};\mathcal{D}_{\rm PL}); \\
	&{\rm s.t.}\ \theta_* \in \mathop{\arg\min}_{\theta} \mathcal{L}_{\rm FT}({\phi}, {\theta};\mathcal{D}_{xy}),
	\vspace{-0.5cm}
\end{align}
where $\mathcal{L}_{\rm PL}$ represents the {\it outer objective}, and for every reward function $r_{\phi}$, $\mathcal{L}_{\rm FT}$ represents the {\it inner objective}. Note that $\theta_*$ denotes the optimal parameters of a policy w.r.t. an RM $r_{\phi}$ for minimizing the inner RL objective. 

\noindent\textbf{Outer objective.}
We desire that the fine-tuned policy ${\theta}_*$, with the help of the reward function $r_{\phi}$, can align with human preferences. In particular, given a pair of answers for a prompt $x$ with preference ranking, $\nu = (y,y',x: y \succ y'|x) \sim \mathcal{D}_{\rm PL}$, we use the Bradley-Terry (BT) model \cite{bradley1952rank} to compute the comparing preference of $p(y \succ y' | x)$ w.r.t. a fine-tuned policy ${\theta_*}$ as follows,
\begin{align}
	\vspace{-0.5cm}
	p_*(y \succ y' | x) = \frac{\pi_{\theta_*}(y|x)}{\pi_{\theta_*}(y|x) + \pi_{\theta_*}(y'|x)}
	\vspace{-0.5cm}
\end{align}

Given a task-specific preference distribution $\mathcal{D}_{\rm PL}$, we apply the maximum likelihood estimation (MLE) to define the {\it outer objective} of preference alignment w.r.t. the policy ${\theta_*}$ as follows,
\begin{align} \label{eq:outer}
	\begin{split}
		&\min_{\phi} \mathcal{L}_{\rm PL}(\phi,\theta_*;\mathcal{D}_{\rm PL}) = \min_{{\phi}} \mathbb{E}_{\nu \sim \mathcal{D}_{\rm PL}}  \left[\ell_{\rm PL}({\phi}, {\theta}_*,\nu) \right] \\
		&{\rm s.t.}\ \ell_{\rm PL}\left( {\phi}, {\theta}_*, \nu \right) = - \log \sigma \left( \log \pi_{\theta_*}(y|x) - \log \pi_{\theta_*}(y'|x) \right)
	\end{split}
\end{align}

\noindent\textbf{Inner objective.}
The inner-loop optimization objective consist with the RL objective of Eq. (\ref{eq:rlhfobj}), and especially the reference policy is defined as the task-specific distribution to avoid the policy drift problem. In particular, given a task-specific distribution $\mathcal{D}_{xy}$, the reference distribution is represented as the conditional task-specific distribution, i.e., $\pi_{\rm ref}(\cdot|x) := \mathcal{D}_{y|x}(\cdot)$ for any $x \sim \mathcal{D}_x$. Then, the RL fine-tuning objective can be represented as
\begin{align}\label{eq:rlhfobj1}
	\vspace{-0.5cm}
	\mathop{\arg\max}_{{\theta}} \mathbb{E}_{x \sim \mathcal{D}_x} \left[ \mathbb{E}_{y \sim \pi_{\theta}(\cdot|x)} \left[ r_{\phi}(x,y) \right] - \beta\mathrm{D}_{\rm KL}\left( \pi_{\theta}(\cdot|x) || \mathcal{D}_{y|x}(\cdot) \right) \right]
	\vspace{-0.5cm}
\end{align}

However, it is impossible to directly optimize the above objective, since the distribution $\mathcal{D}_{y|x}(\cdot)$ is unknown. Thus, we consider an approximate solution to this objective. The derivation largely follows Peters and Schaal \cite{peters2007reinforcement} and Peng et al. \cite{peng2019advantage}. The major observation is that the objective is strictly concave and thus we can use the Karush-Kuhn-Tucker (KKT) conditions \cite{boyd2004convex} to obtain a globally optimal solution to Eq. (\ref{eq:rlhfobj1}), such that for any $x,y$,
\begin{align} \label{eq:policyobj}
	\vspace{-0.5cm}
	\pi_*(y|x) = \frac{1}{Z(x)} \mathcal{D}_{y|x}(y)\exp\left(\frac{1}{\beta}r_{\phi}(x,y)\right), 
	\vspace{-0.5cm}
\end{align}
where $Z(x)$ represents the partition function and $\beta$ represents the {\it reward controlling factor}. A detailed derivation is available in Appendix \ref{apdx:paraest}. We focus on a parameterized estimator (e.g., a neural network) to tackle a parameter estimation problem $\min_{{\theta}} \mathbb{E}_{x  \sim \mathcal{D}_{x}}[\mathrm{D}_{\rm KL}(\pi_*(\cdot|x)||\pi_{\theta}(\cdot|x))]$. Thereby, through sampling $z=(x,y) \sim \mathcal{D}_{xy}$, the inner objective can be represented by logical regression,
\begin{align} \label{eq:innerobj}
	\vspace{-0.5cm}
	\begin{split}
		\mathop{\arg\min}_{\theta} \mathcal{L}_{\rm FT}(\phi, \theta, \mathcal{D}_{xy}) &=  \mathop{\arg\min}_{\theta} \mathbb{E}_{x \sim \mathcal{D}_x} \mathbb{E}_{y \sim \mathcal{D}_{y|x}} \left[ \ell_{\rm FT}({\phi}, {\theta}, z)  \right] \\
		\qquad\qquad\qquad\qquad\quad&= \mathop{\arg\min}_{\theta} \mathbb{E}_{z \sim\mathcal{D}_{xy}} \left[ \ell_{\rm FT}({\phi}, {\theta}, z) \right] \\
		{\rm s.t.}\ \ell_{\rm FT}({\phi}, {\theta}, z) &= - \log \pi_{\theta}(y|x) \exp \left( \frac{1}{\beta}r_{\phi}(x,y) \right)
	\end{split}
\vspace{-0.5cm}
\end{align}

\subsection{Meta-Learning for Out-of-Distribution Preference Learning} \label{sec:metalearn}
\begin{wrapfigure}[13]{r}{0.5\textwidth}
	\vspace{-0.5cm}
	\centering
	\includegraphics[width=1.\linewidth]{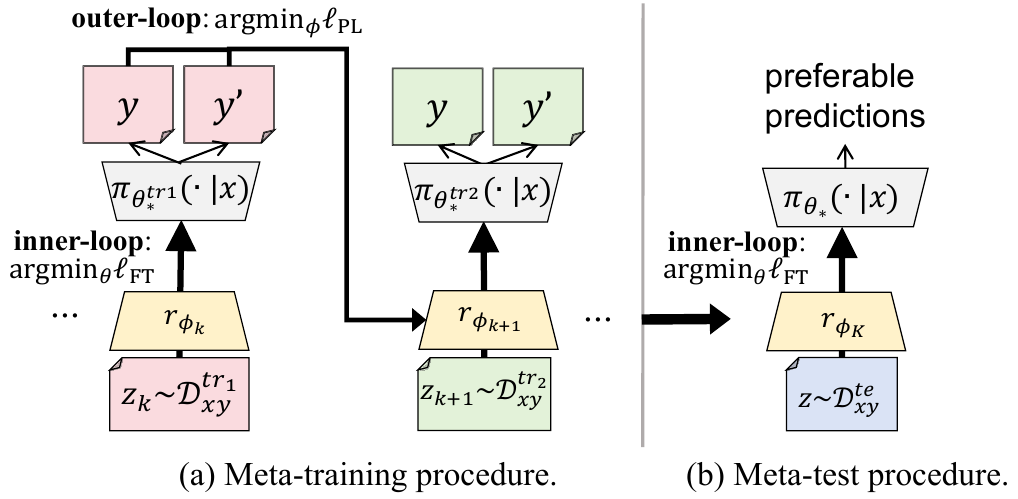}
	\caption{Overall training process. Meta-training optimizes the RM $\Phi$ using $K$ iterations of SGD. During testing, the RM $\Phi_{K}$ is utilized to optimize the test policy $\theta_*$.}
	\label{fig:overall}
\end{wrapfigure}
We generalize the ID bilevel optimization framework to tackle OOD preference learning. To this end, we assume that there exists a meta-distribution $\mathcal{D}_{\mathcal{T}}$ over the set of tasks $\mathcal{T}$. Previous studies on the distribution generalization refer to meta-learning \citep{baxter2000model,finn2017model}.

To learn a generalized reward modeling for RL fine-tuning on unseen test distributions, we utilize a meta-learning framework, where the {\it outer objective} described in the \S \ref{sec:bilevelrm} should be the extent to compute the preference loss across tasks, and the {\it inner objective} can be unchangeable for each distribution. In particular, given a task $T \in \mathcal{T}$, there exists a task-specific distribution $\mathcal{D}^T_{xy}$. For preference learning, each training distribution has a preference distribution $\nu = (x,y,y') \sim \mathcal{D}^T_{\rm PL}$. The optimal variable in {\it inner objective} should be the task-specific policy ${\theta^T_*}$. A key point is that the RM $r_{\phi}$ is shared across all task distributions. With this notation, the inner and outer objectives can be represented as follows,
\begin{align} \label{eq:metaobjouter}
	\vspace{-0.5cm}
	&\min_{\phi} {\bf L}(\phi) = \min_{\phi} \mathbb{E}_{T \sim \mathcal{D}_{\mathcal{T}}} \mathbb{E}_{\nu \sim \mathcal{D}^T_{\rm PL}} \left[ \ell_{\rm PL}\left( \phi, \theta^{T}_*, \nu \right)\right] \\
	&{\rm s.t.}\ \theta^{T}_* \in \mathop{\arg\min}_{{\theta^T}} \mathbb{E}_{z \sim \mathcal{D}^T_{xy}} \left[ \ell_{\rm FT}({\phi}, {\theta^T}, z) \right],\label{eq:metaobjinner}
	\vspace{-0.5cm}
\end{align}
where the objective $\ell_{\rm PL}(\phi, {\theta^T_*},\nu)$, $\nu\sim\mathcal{D}^T_{\rm PL}$ represents the preference learning loss, which has a same formulation as the ID version of Eq. (\ref{eq:outer}), $\ell_{\rm FT}(\phi, \theta^T, z)$, $z\sim \mathcal{D}^T_{xy}$ represents the fine-tuning loss, which has the same formulation as the ID version of Eq. (\ref{eq:innerobj}).

\subsection{Gradient-Based Algorithm} \label{sec:algorithm}
\RestyleAlgo{ruled}

\SetKwComment{Comment}{/* }{ */}
\vspace{-0.5cm}
\begin{algorithm}[h]
	\footnotesize
	\caption{Meta-training procedure}\label{alg:meta-train}
	\KwIn{Meta-training data $\{\mathcal{D}^T_{\rm PL}, \mathcal{D}^T_{xy}\}_{T \in \mathcal{T}_{tr}}$, RM $\phi$, policy $\theta$}
	\KwOut{Optimal RM $\phi_*$ for meta-test in Alg. \ref{alg:meta-test}}
	\For{$k = 0,1,2, \ldots, K-1$}{
		Randomly draw a sample task $T \in \mathcal{T}_{tr}$ from the meta-distribution $\mathcal{D}_{\mathcal{T}}$ \\
		Initialize $\theta$ with SFT parameters on $\mathcal{D}^T_{xy}$ \\
		\For{$t = 0,1,2, \ldots, D-1$}{
			Randomly draw a sample $z_{k,t}$ from $\mathcal{D}^T_{xy}$\\
			Update $\theta_{k,t+1} \gets \theta_{k,t} - \alpha \nabla_{\theta} \ell_{\rm FT} (\phi_{k}, \theta_{k,t}; z_{k,t})$$\hfill \rhd$ minimize the inner obj. of Eq. (\ref{eq:metaobjinner}) \\ 
		}
		Randomly draw a sample $\nu_k$ from $\mathcal{D}^T_{\rm PL}$\\
		Update $\phi_{k+1} \gets \phi_{k} - \eta \nabla_{\phi} \ell_{\rm PL}(\phi_{k}, \theta_{k,D}; \nu_k)$ $\hfill \rhd$ minimize the outer obj. of Eq. (\ref{eq:metaobjouter})
	}
\end{algorithm}
\vspace{-1.5cm}
\begin{algorithm}[h]
	\footnotesize
	\caption{Meta-test procedure}\label{alg:meta-test}
	\KwIn{Meta-test data $\mathcal{D}^{te}_{xy}$; optimized reward function $\phi_*$ from Alg. \ref{alg:meta-train}}
	\KwOut{Optimal policy $\theta_*$ for testing}
	Initialize $\theta$ with SFT parameters on $\mathcal{D}^{te}_{xy}$ \\
	\While{not converge or not reach stopping conditions}{
		Randomly draw a sample $z$ from $\mathcal{D}^{te}_{xy}$\\
		Update $\theta \gets \theta - \alpha \nabla_{\theta} \ell_{\rm FT}(\phi_*, \theta; z)$ $\hfill \rhd$ minimize the inner obj. of Eq. (\ref{eq:metaobjinner})
	}
\end{algorithm}
\vspace{-0.5cm}
Following the previous work on gradient-based bilevel programming \citep{finn2017model,ji2021bilevel}, we use a gradient-based approach to tackle the bilevel optimization objective in Eq. (\ref{eq:metaobjouter}) and Eq. (\ref{eq:metaobjinner}). The whole training process consists of meta-training and meta-test procedures, as shown in Fig. \ref{fig:overall}.

\noindent\textbf{Meta-training.} 
As illustrated in Alg. \ref{alg:meta-train}, for each outer-loop iteration $k \in \{0,1,\ldots, K-1\}$, the inner-loop process runs $D$ steps of stochastic gradient decent (SGD) to obtain $\theta_{k,D}$. Then, the outer-loop process uses SGD optimization w.r.t. $\phi_k$ via computes the gradient $\mathbb{E} \left[ \frac{\partial \ell_{\rm PL}(\phi_k, \theta_{k,D})}{\partial \phi_k} \right]$. Specifically, we compute the explicit form of the $\frac{\partial \ell_{\rm PL}(\phi_k, \theta_{k,D})}{\partial \phi_k}$ in the following proposition. After $K$ steps of meta-training, Alg. \ref{alg:meta-train} outputs optimized parameters of the reward function $\phi_*$.
\begin{proposition}\label{pro:gradient}
	\small
	For any outer-loop step $k \in \{0,1,2,\ldots,K-1\}$, with the outer-loop input $z_k$ and the inner-loop inputs $\{z_{k,t}\}_{t=0}^{D-1}$, the gradient $\frac{\partial \ell_{\rm PL}(\phi_k, \theta_{k,D})}{\partial \phi_k}$ takes the analytical form,
	\begin{align*}\footnotesize
		\vspace{-0.5cm}
		\begin{split}
			\frac{\partial \ell_{\rm PL}(\phi_k, \theta_{k,D})}{\partial {\phi}_k} =
			-\frac{\alpha}{\beta} \nabla_{{\theta}} A_k \left(1 - \sigma \left( A_k \right) \right) \sum_{t=0}^{D-1} \nabla_{\phi} H_{k,t} \nabla_{{\theta}}^{\top} F_{k,t} R_{k,t} \prod_{j=t+1}^{D-1} \left( I + \alpha \nabla_{{\theta}}^2 F_{k,t} R_{k,t} \right),
		\end{split}
		\vspace{-0.5cm}
	\end{align*}
	where $A_k := \log \pi_{\theta_{k,D}}(y_k|x_k) -  \log \pi_{{\theta}_{k,D}}(y'_k|x_k)$, $R_{k,t} := \exp \left( \frac{1}{\beta}r_{\phi_k}(x_{k,t},y_{k,t}) \right)$, $F_{k,t} := \log \pi_{\theta_{k,t}}(y_{k,t}|x_{k,t})$ and $H_{k,t} := r_{\phi_k}(x_{k,t},y_{k,t})$.
\end{proposition}
\begin{proof}
	The major proof steps consist of the chain rule of partial differentiation, and the detailed derivation is available in Appendix \ref{apdx:proproof}.
\end{proof}

\noindent\textbf{Meta-test.}
As illustrated in Alg. \ref{alg:meta-test}, the meta-test procedure takes the optimal RM $\phi_*$ by meta-training to learn a policy for a given test distribution. In particular, the meta-test process computes policy gradients based on the optimal RM $\phi_*$ and uses SGD to optimize the task-specific policy with a sufficient number of iterations for convergence. The meta-test procedure can also be viewed as test-phase adaptation.

\subsection{Convergence Analysis} \label{sec:convergence}
Following previous work on convergence analysis of SGD \citep{drori2020complexity,ji2021bilevel}, we study the convergence rate of the proposed bilevel programming to the stationary point. In particular, we aim to show how the proposed algorithm minimizes $\mathbb{E}[\Vert \nabla_{\phi} {\bf L}(\overline{\phi}) \Vert^2]$, where ${\bf L}(\phi)$ is defined in Eq. (\ref{eq:metaobjouter}) and $\overline{\phi}$ denotes the output of algorithm. 
Following previous work on the convergence rate of gradient-based bilevel optimization \citep{ji2021bilevel}, we focus on the following types of basic loss functions for PL.
\begin{assumption} \label{asp:strong1}
	\small
	The maximum likelihood estimation (MLE) objective for policy learning, $- \log \pi_{\theta}$, is $\mu$-strongly convex. 
\end{assumption}

Besides, we take the following standard Lipschitz assumptions on the basic MLE loss functions of policy $-\log\pi_{\theta}$ and the reward function $r_{\phi}$, which are largely followed by those that have been widely adopted in bilevel optimization \citep{ghadimi2018approximation,ji2021bilevel}. 
\begin{assumption} \label{asp:lip2}
	\small
	The loss function $- \log \pi_{\theta}$ and reward function $r_{\phi}$ satisfy
	(1) $-\log\pi_{\theta}(y|x)$ is $M$-Lipschitz.
	(2)$-\nabla \log \pi_{\theta}(y|x)$ is $L$-Lipschitz.
	(3) $-\nabla^2 \log \pi_{\theta}(y|x)$ is $\rho$-Lipschitz.
	(4) $r_{\phi}$ is $T$-Lipschitz.
	(5) $\nabla r_{\phi}(x,y)$ is $P$-Lipschitz. More specifically,
	\begin{itemize}
		\item $-\log\pi_{\theta}(y|x)$ is $M$-Lipschitz, i.e., for any $\theta, \theta'$ and any $(x,y)$,
		\begin{align}
			\left\vert \log\pi_{\theta}(y|x) - \log\pi_{\theta'}(y|x) \right\vert \leq M \Vert \theta - \theta' \Vert 
		\end{align}
		\item $-\nabla \log \pi_{\theta}(y|x)$ is $L$-Lipschitz, i.e., for any $\theta, \theta'$ and any $(x,y)$,
		\begin{align}
			\Vert \nabla \log \pi_{\theta}(y|x) - \nabla \log \pi_{\theta'}(y|x) \Vert \leq L \Vert \theta - \theta' \Vert
		\end{align}
		\item $-\nabla^2 \log \pi_{\theta}(y|x)$ is $\rho$-Lipschitz, i.e., for any $\theta, \theta'$ and any $(x,y)$,
		\begin{align}
			\Vert \nabla^2 \log \pi_{\theta}(y|x) - \nabla^2 \log \pi_{\theta'}(y|x) \Vert \leq \rho \Vert \theta - \theta' \Vert
		\end{align}
		\item $r_{\phi}$ is $T$-Lipschitz, i.e., for any $\phi, \phi'$ and any $(x,y)$,
		\begin{align}
			\left\vert r_{\phi}(x,y) - r_{\phi'}(x,y) \right\vert \leq T \Vert \phi - \phi' \Vert
		\end{align}
		\item $\nabla r_{\phi}(x,y)$ is $P$-Lipschitz, i.e., for any $\phi, \phi'$ and any $(x,y)$,
		\begin{align}
			\left\Vert \nabla r_{\phi}(x,y) - \nabla r_{\phi'}(x,y) \right\Vert \leq P \Vert \phi - \phi' \Vert	
		\end{align}
	\end{itemize}
\end{assumption}

We also consider the following reasonable assumptions on the boundness of reward function and policy prediction probability.
\begin{assumption} \label{asp:bound3}
	\small
	The function $-\log \pi_{\theta}(y|x)$ and $\exp(r_{\phi}(x,y))$ are bounded, i.e., for any $(x,y)$, there exist $b,B,C > 0$,$b \leq \exp(r_{\phi}(x,y)) \leq B, \vert \log \pi_{\theta}(y|x) \vert \leq C$.
\end{assumption}

As typically adopted in the analysis of stochastic optimization by Ji et al. \cite{ji2021bilevel}, we make the following assumptions of bounded variances for the loss functions.
\begin{assumption} \label{asp:var4}
	\small
	Gradients $\nabla_{\theta} \ell_{\rm FT}({\phi}, {\theta};z)$, $\nabla_{\phi}\nabla_{\theta} \ell_{\rm FT}(r_{\phi}, \pi_{\theta};z)$, $\nabla^2_{\theta} \ell_{\rm FT}(r_{\phi}, \pi_{\theta};z)$ and $\frac{\partial \ell_{\rm PL}(r_{\phi}, \pi_{{\theta}};z)}{\partial \phi}$ have bounded variances of $\sigma$,$\sigma^2_1$,$\sigma^2_2$,$\sigma^2_3$, respectively. More specifically,
	\begin{align}
		\mathbb{E}_{z} \left[ \Vert \nabla_{\theta}\mathcal{L}_{\rm FT}({\phi}, {\theta}) - \nabla_{\theta} \ell_{\rm FT}({\phi}, {\theta};z) \Vert^2 \right] \leq& \sigma \\
		\mathbb{E}_{z} \left[ \Vert \nabla_{\phi}\nabla_{\theta}\mathcal{L}_{\rm FT}({\phi}, {\theta}) - \nabla_{\phi}\nabla_{\theta} \ell_{\rm FT}({\phi}, {\theta};z) \Vert^2 \right] \leq& \sigma^2_1 \\
		\mathbb{E}_{z} \left[ \nabla^2_{\theta} \mathcal{L}_{\rm FT}({\phi}, {\theta}) - \nabla^2_{\theta} \ell_{\rm FT}({\phi}, {\theta};z) \right] \leq& \sigma^2_2 \\
		\mathbb{E}_{\nu}\left[\left\Vert \frac{\partial \mathcal{L}_{\rm PL}({\phi}, {\theta})}{\partial \phi} -  \frac{\partial \ell_{\rm PL}({\phi}, {\theta};\nu)}{\partial \phi}  \right\Vert^2 \right] \leq& \sigma^2_3
	\end{align}
\end{assumption}

Then, we present the main convergence result based on the above assumptions.
\begin{theorem} \label{thm:1}
	\small
	Suppose Assumption \ref{asp:strong1}, \ref{asp:lip2}, \ref{asp:bound3} and \ref{asp:var4} hold, choose the iteration number $D$ in the inner-loop and the reward controlling factor $\beta$ satisfy the conditions: (i) {\small$D \geq - \frac{\log \left(\frac{4M^2}{(2L+M^2)^2}  +  \left\Vert \theta_{0} - \theta_* \right\Vert^2\right)}{\log \left(3\left(1-\frac{\alpha b\mu}{\exp\beta}\right)^{2}\right)}$, and (ii) {\small$\beta \leq \min \Big\{ \log \frac{4\alpha B^2(L^4+M^2\rho^2)}{b\mu (2L+M^2)^2}, \log\left( \alpha b\mu /  1 - \sqrt{\frac{1-\kappa}{3}}\right) \Big\}$}, for some coefficient $\kappa > 0$}, then we have
	\begin{align}
		\vspace{-0.5cm}
		\frac{1}{K}\sum_{k=0}^{K-1} \mathbb{E} \left[ \Vert \nabla {\bf L}(\phi_{k})\Vert^2 \right] \leq \mathcal{O}\left(\frac{1}{K} + \frac{1}{\beta^2} \right) + \tau,
		\vspace{-0.5cm}
	\end{align}
	where $\tau$ denotes a $K$, $\beta$-unrelated term, which has the form of $ \tau = 3\sigma_3^3 + \frac{144 M^2 \alpha^2 \sigma_1^2}{\kappa} $. 
\end{theorem}
\begin{proof}
	See Appendix \ref{apdx:proofthm1} for detailed derivations.
\end{proof}
\begin{remark}
	Theorem \ref{thm:1} shows that the proposed bilevel optimization algorithm converges w.r.t. the number of outer-loop iteration $\mathcal{O}\left(\frac{1}{K}\right)$ and the controlling factor of reward function $\mathcal{O}\left(\frac{1}{\beta^2}\right)$. This suggests that (a) more iterations of SGD and (b) a larger reward controlling factor under the condition (ii) makes the algorithm converge better.
\end{remark}

\subsection{Experimental Setup}
\noindent\textbf{Tasks.}
For the {\bf controlled sentiment generation (CSG)} task, $x$ represents a prefix of a review from the Amazon Review dataset, which contains four domains \citep{blitzer2007biographies}. The policy generates the remaining review context $y$ with a positive sentiment. We use GPT-4 \citep{achiam2023gpt} to generate pairwise reviews for each prefix. Following Li et al. \cite{li2017deeper}, we adopt a leave-one-out protocol for OOD evaluation. Specifically, each experiment on Amazon Review is trained on three of the four distributions and evaluated on the remaining distribution.
For the {\bf knowledge answer generation (KAG)} task, we utilize the Stanford Human Preferences Dataset (SHP) \citep{ethayarajh2022understanding}, which comprises 18 domains of knowledge question-answering (QA) with preference labels. Accordingly, $x$ is a Reddit question, the policy generates a proper response $y$ to the question. We employ a cross-validation evaluation method, i.e., randomly dividing the 18 domains of the SHP dataset into six splits, one split is used for testing, while the remaining splits are used for training in each experiment. 

\noindent\textbf{Metrics.}
We use three metrics for evaluation. To assess the effectiveness of policy learning in aligning with human preferences, we introduce a metric called PL accuracy $\mathcal{A}_{\rm PL}$, which quantifies the proportion of correctly predicted comparisons by the learned policy in the preference test set.
\begin{align}
	\vspace{-0.5cm}
	\mathcal{A}_{\rm PL}(\theta;\mathcal{D}_{\rm PL}) = \frac{1}{|\mathcal{D}_{\rm PL}|} \sum_{(y, y',x: y \succ y'|x)\in \mathcal{D}_{\rm PL}} \mathbb{I}\left[\frac{\log\pi_{\theta}(y|x)}{|y|} > \frac{\log\pi_{\theta}(y'|x)}{|y'|} \right],
	\vspace{-0.5cm}
\end{align}
where the normalization terms $1/|y|$ and $1/|y'|$ are incorporated to mitigate the influence of generation length on the computation of language probability. In the CSG task, we assess the quality of generation for preference alignment using a {\bf learned RM}. Specifically, we fine-tune a RoBERTa$_{\rm LARGE}$ model \citep{liu2019roberta} on the Amazon Review dataset to serve as the ground truth RM for scoring the generated reviews. In the KAG task, as the ground truth reward function is unknown, we evaluate approaches based on their {\it win rate} against the preferred responses in the test set. We use {\bf GPT-4 judgement} to assess response correctness and usefulness.

\noindent\textbf{Optimization Protocol.}
For model selection, we select the best RM $\phi_*$ learned in the meta-training phase based on the minimum PL accuracy on the evaluation sets, and choose the best policy model $\theta_*$ learned in the meta-test phase based on the minimum training loss. The hyperparameters used in our experiments largely follow previous works \citep{ramamurthy2022reinforcement,ouyang2022training}. Additionally, we select hyperparameters related to bilevel programming from a preset range based on development experiments. For example, the inner-loop learning rate $\alpha$ is chosen from ${ 5e^{-6}, 1e^{-5}, 3e^{-5}, 5e^{-5} }$, the reward controlling factor $\beta$ is selected from ${0.1,1.0,2.0,4.0,8.0}$, and the inner-loop iteration number $D$ is selected from ${ 50, 100, 150, 200 }$.

\noindent\textbf{Approaches.}
Our methods are built upon the {Base Model} Flan-T5 XXL \citep{chung2022scaling}, comprising more than 11B parameters. To expedite the training procedure, we adopt the parameter-efficient fine-tuning method LoRA \citep{hu2021lora}. Specifically, the trainable parameters of an RM consist of a LoRA adapter and a value head, while the trainable parameters of a policy consist of a LoRA adapter and a LM head. We also consider several strong baselines of OOD PL for comparison. For {SFT}, we utilize supervised data from both the training and test distributions for policy fine-tuning. The other approaches use SFT parameters for model initialization. For a fair comparison with PPO \citep{schulman2017proximal}, we use preference data from all training distributions to train the RM with a multi-task learning approach. We also consider two RLHF approaches: {PPO (tr)} and {PPO (te)} use supervised data from the training distributions and test distribution for policy fine-tuning, respectively. For {DPO} \citep{rafailov2024direct}, we use preference data from the training distributions for policy training. Additionally, we also consider a variation of our approach named {HPL}, which directly optimizes the PL objective $\mathcal{L}_{\rm PL}$ using preference data from the training distributions.

\subsection{Evaluation with Preference Data}
\begin{figure}[h]
	\centering
	\subfigure[{\tt book}.]{
		\begin{minipage}[h]{0.23\textwidth}
			\centering
			\includegraphics[width=3cm]{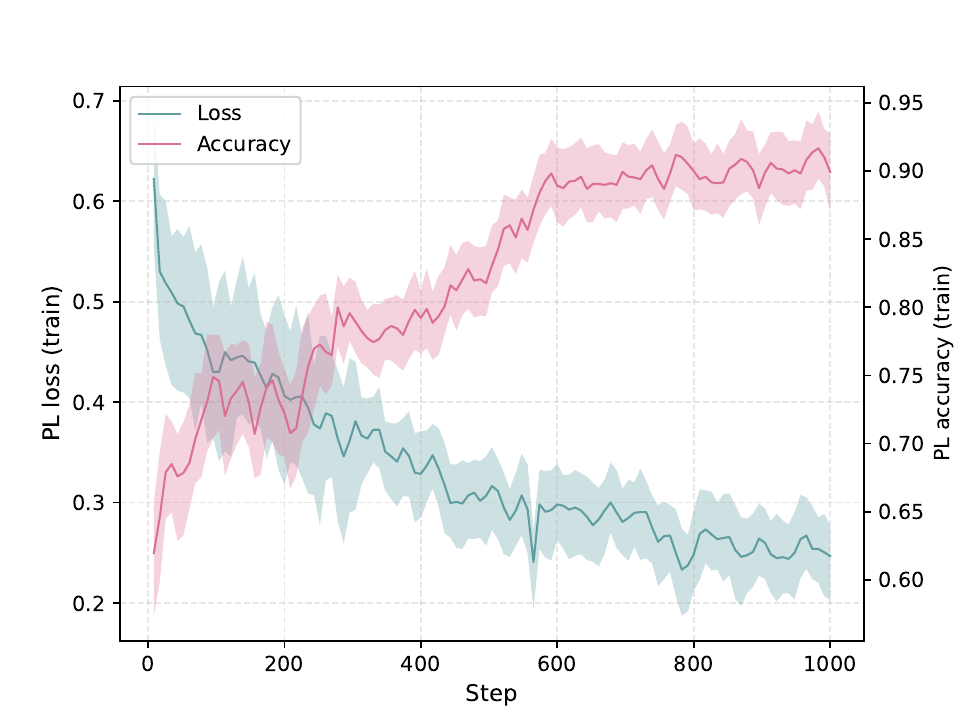}
	\end{minipage}}
	\subfigure[{\tt dvd}.]{
		\begin{minipage}[h]{0.23\textwidth}
			\centering
			\includegraphics[width=3cm]{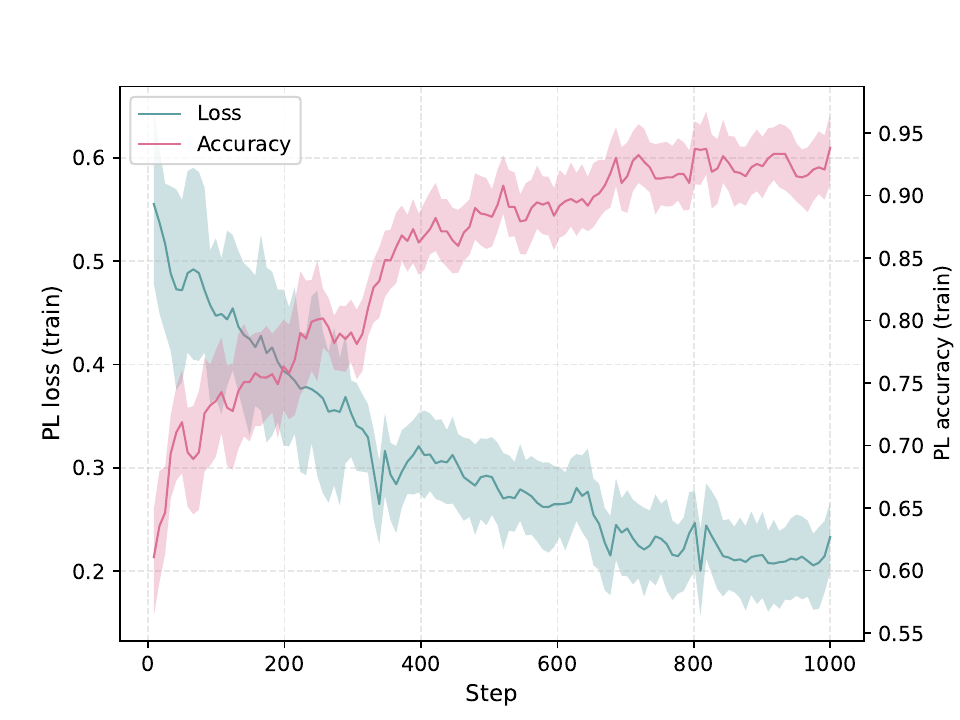}
	\end{minipage}}
	\subfigure[{\tt electronics}.]{
		\begin{minipage}[h]{0.23\textwidth}
			\centering
			\includegraphics[width=3cm]{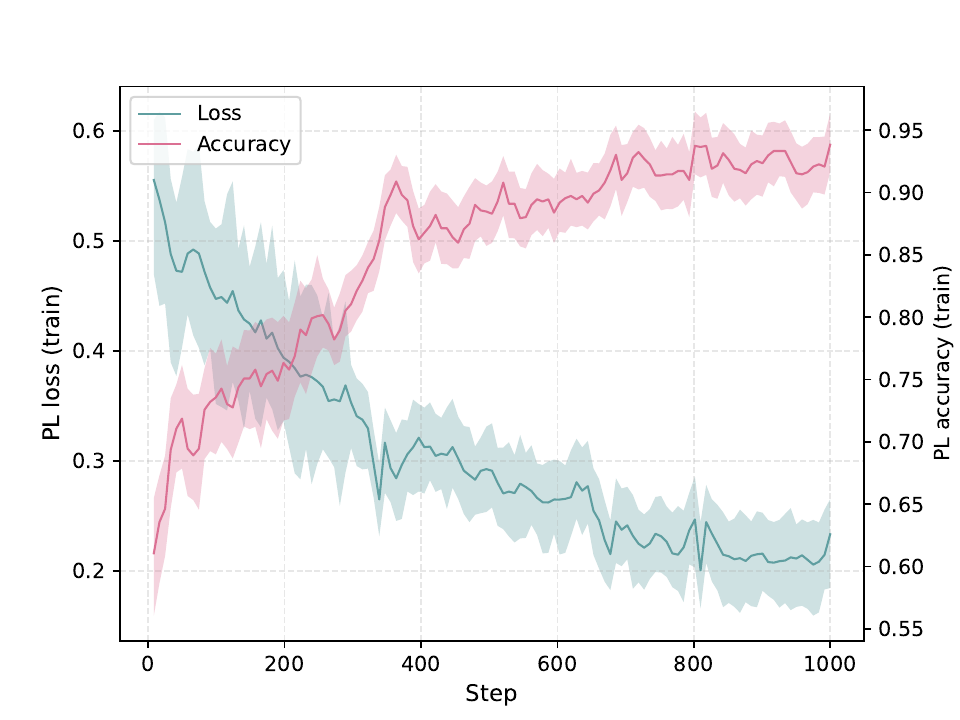}
	\end{minipage}}
	\subfigure[{\tt kitchen}.]{
		\begin{minipage}[h]{0.23\textwidth}
			\centering
			\includegraphics[width=3cm]{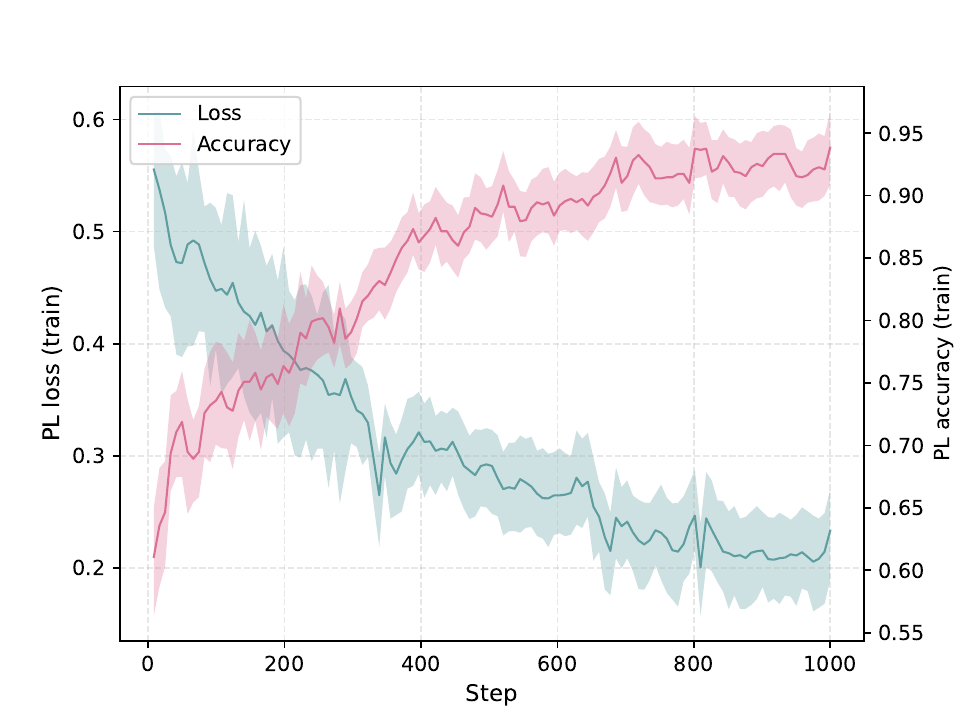}
	\end{minipage}}
	\subfigure[{\tt book}.]{
		\begin{minipage}[h]{0.23\textwidth}
			\centering
			\includegraphics[width=3cm]{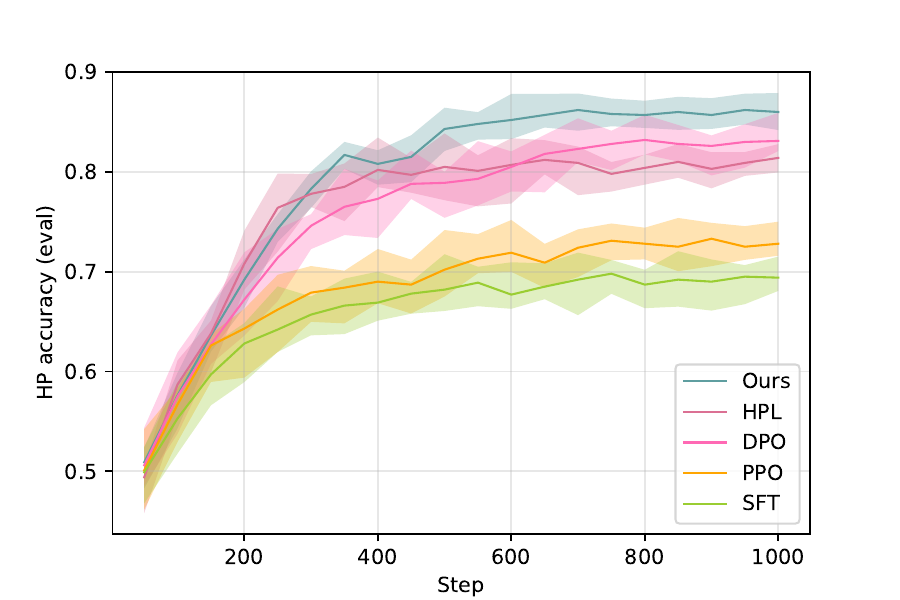}
	\end{minipage}}
	\subfigure[{\tt dvd}.]{
		\begin{minipage}[h]{0.23\textwidth}
			\centering
			\includegraphics[width=3cm]{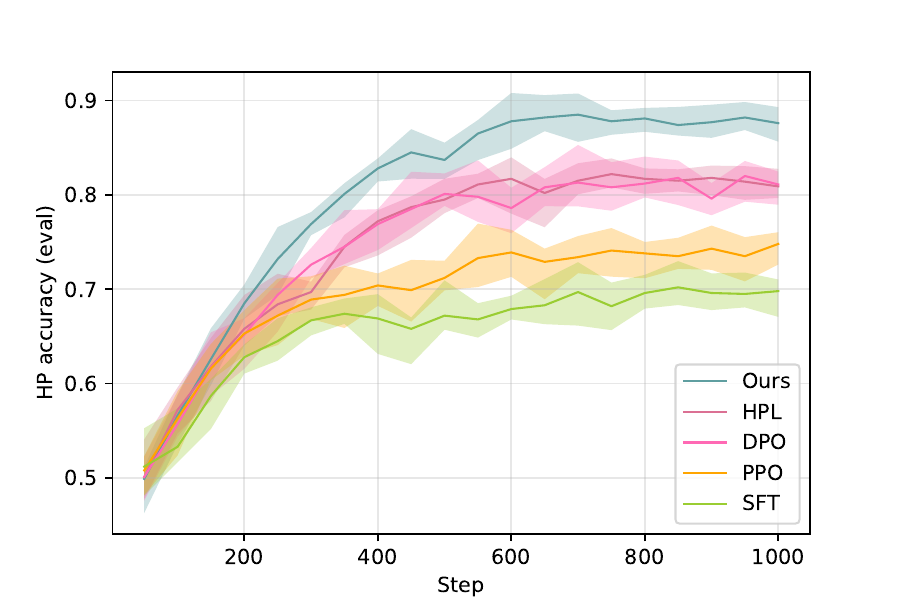}
	\end{minipage}}
	\subfigure[{\tt electronics}.]{
		\begin{minipage}[h]{0.23\textwidth}
			\centering
			\includegraphics[width=3cm]{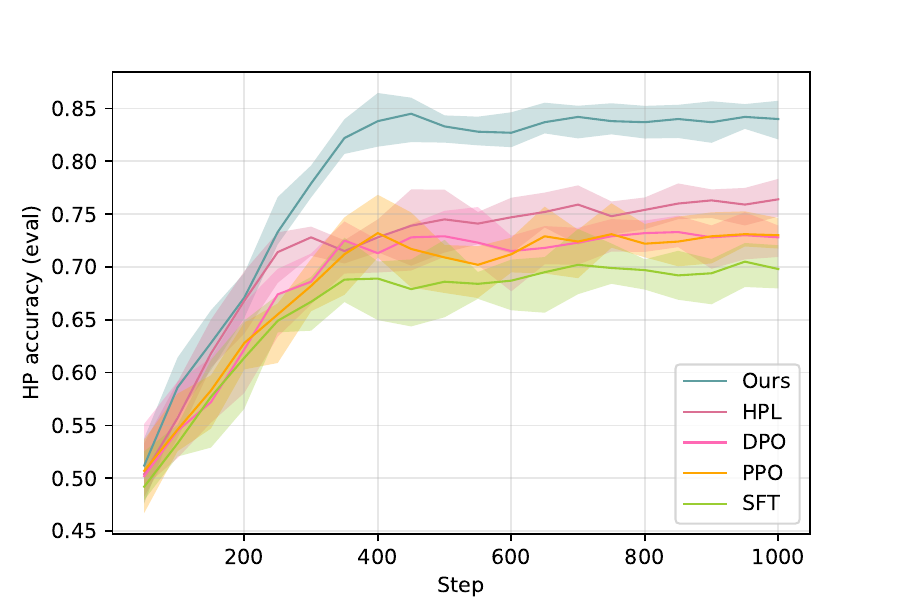}
	\end{minipage}}  
	\subfigure[{\tt kitchen}.]{
		\begin{minipage}[h]{0.23\textwidth}
			\centering
			\includegraphics[width=3cm]{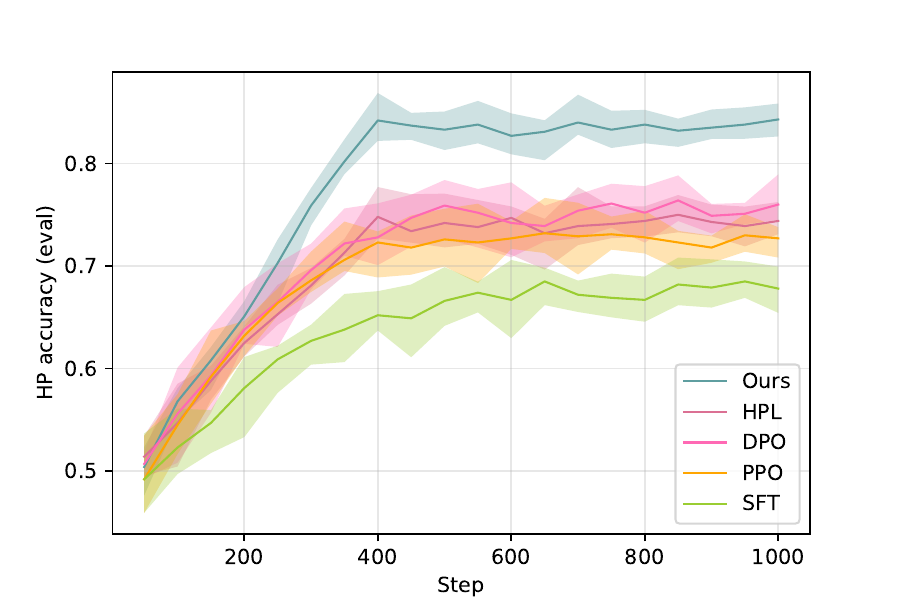}
	\end{minipage}}
	\caption{(a)-(d) illustrate the training loss and PL accuracy $\mathcal{A}_{\rm PL}$ domain, and (e)-(h) illustrate the evaluation accuracy against the training steps for each held-out CSG distribution.} \label{fig:pl-step}
\end{figure}

Preference data from the test distributions is utilized for evaluation. Specifically, we examine the optimization of the outer objective against the training steps, as depicted in Fig. \ref{fig:pl-step}.  (a) and (c) demonstrate that the PL loss $\ell_{\rm PL}$ decreases rapidly within the initial 600 steps for the CSG task and within the initial 3,000 steps for the KAG task, respectively. Subsequently, the PL loss reaches a low value for both tasks, exhibiting only minor fluctuations. The PL accuracy $\mathcal{A}_{\rm PL}$ increases in the opposite trend to the PL loss and eventually surpasses $0.9$ for both tasks. This indicates that our method is effective in conducting ID PL. To compare various methods for OOD PL, Fig. \ref{fig:pl-step} (b) and (d) depict the PL accuracy $\mathcal{A}_{\rm PL}$ on the evaluation set of test distributions for CSG and KAG, respectively. In comparison to the SFT baseline, other methods utilize preference data for training, resulting in higher PL accuracy for both tasks. Moreover, our method optimizes a bilevel optimization objective for OOD PL, significantly outperforming other methods by achieving a PL accuracy of $0.86$ for CSG and $0.75$ for KAG.

\subsection{Evaluation with a Learned RM}
\begin{figure}[h]
	\centering
	\subfigure[{\tt book}.]{
		\begin{minipage}[h]{0.23\textwidth}
			\centering
			\includegraphics[width=3cm]{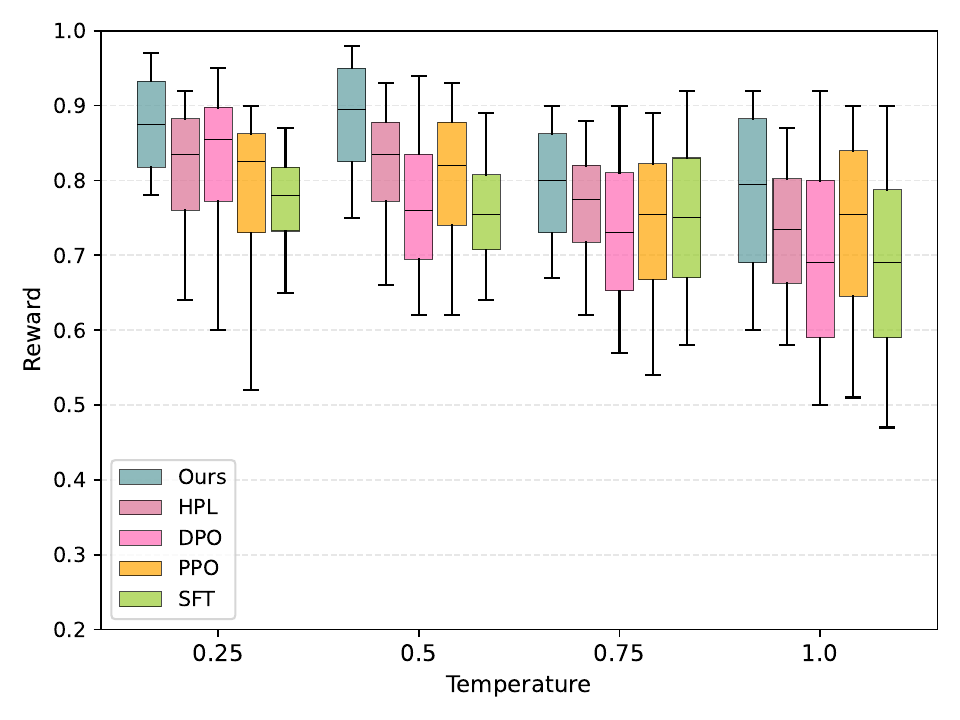}
	\end{minipage}}
	\subfigure[{\tt dvd}.]{
		\begin{minipage}[h]{0.23\textwidth}
			\centering
			\includegraphics[width=3cm]{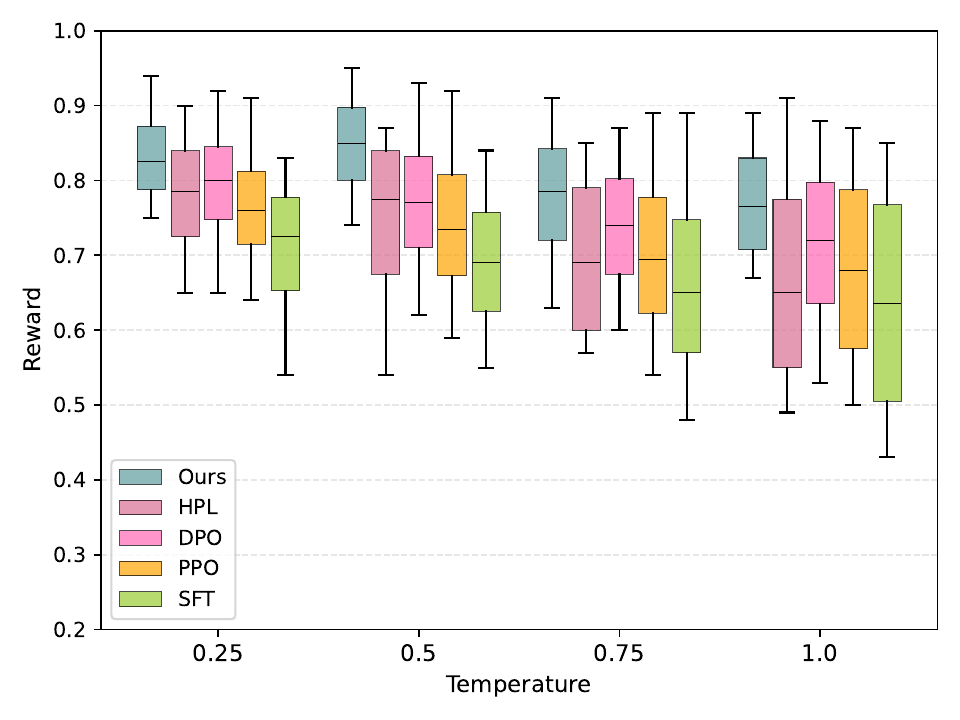}
	\end{minipage}}
	\subfigure[{\tt electronics}.]{
		\begin{minipage}[h]{0.23\textwidth}
			\centering
			\includegraphics[width=3cm]{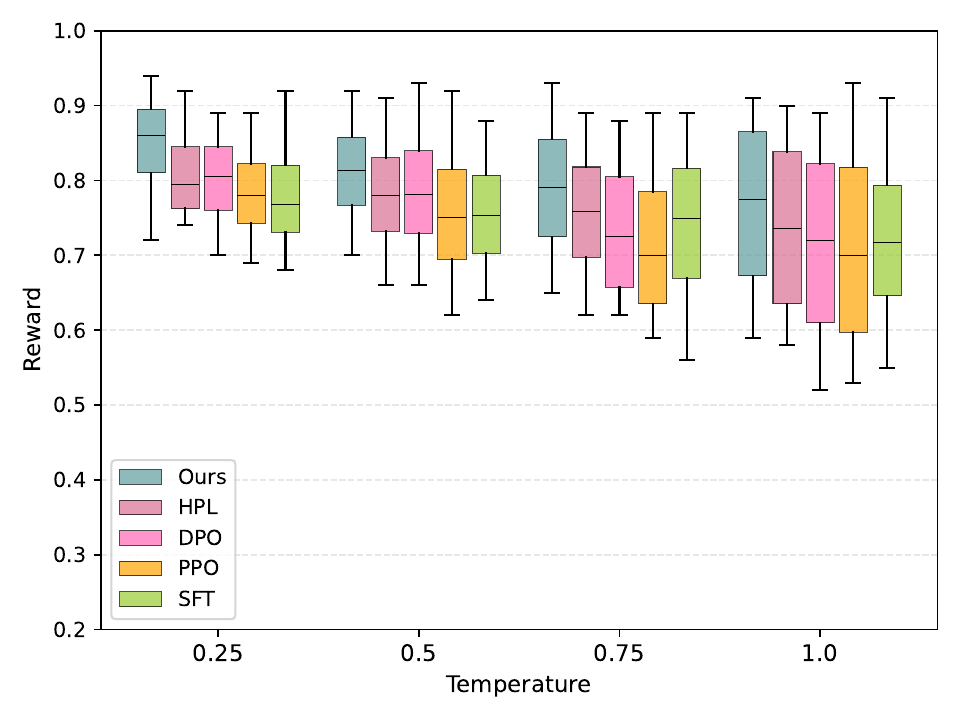}
	\end{minipage}}
	\subfigure[{\tt kitchen}.]{
		\begin{minipage}[h]{0.23\textwidth}
			\centering
			\includegraphics[width=3cm]{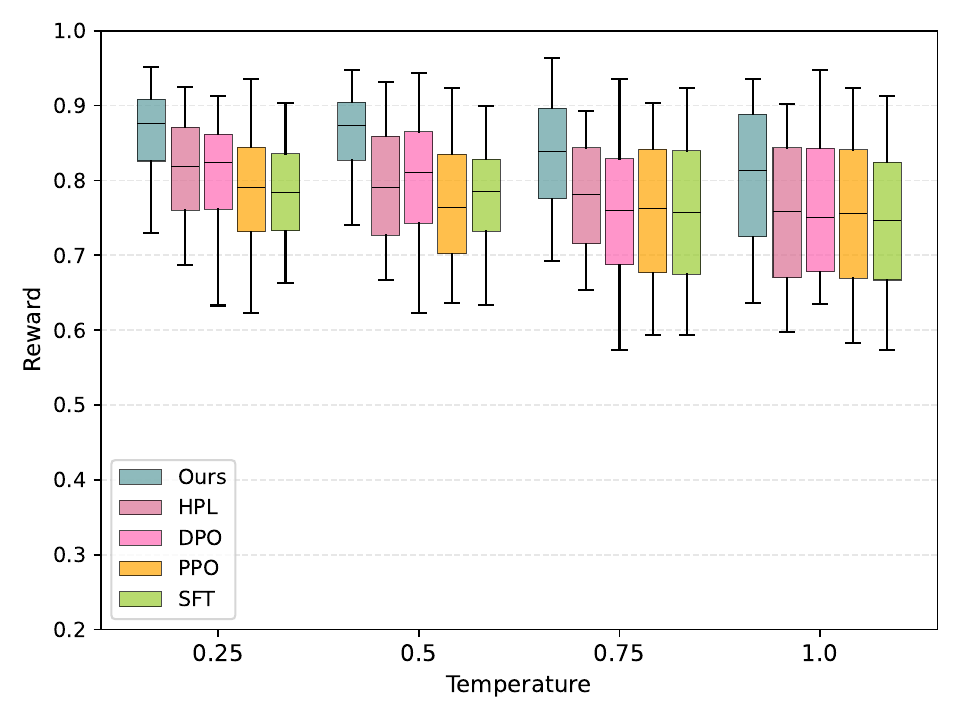}
	\end{minipage}}    
	\caption{Reward on four held-out Amazon Review distributions.} \label{fig:rewardamazon}
\end{figure}
Following the approach of Rafailov et al. \cite{rafailov2024direct}, we employ a learned RM to evaluate PL in CSG. Specifically, we fine-tune a RoBERTa$_{\rm LARGE}$ model on four domains of the Amazon Review dataset to serve as the reference reward function. As depicted in Fig. \ref{fig:rewardamazon}, we compare various OOD PL methods across four held-out domains, using different temperatures ranging from ${0.25, 0.5, 0.75, 1.0}$ to control the randomness of generation. The results for each temperature represent the average of three repetitions of generation. From Fig. \ref{fig:rewardamazon}, it is observed that as the temperature increases from $0.25$ to $1.0$, the variances of rewards for all methods tend to increase across the four held-out domains. Despite the wide range of rewards observed for all methods, our method consistently outperforms other baselines in terms of median reward across all temperatures. This demonstrates the effectiveness of our proposed approach in optimizing a general RM for OOD PL.

\subsection{Evaluation with GPT-4 Judgement}
\begin{table*}[h] 
	\centering
	\caption{GPT-4 evaluation of win rate on the KAG. We format {\bf the best}, \underline{the second best} and \textcolor{gray}{worse than SFT} results.}
	\scalebox{.69}{
		\begin{tabular}{lcccccccccccccccccccccccc}
			\toprule
			\multirow{2}{*}{\textbf{Method}}& \multicolumn{4}{c}{{\tt ac,an,ba}} & \multicolumn{4}{c}{{\tt ca,cu,do}} & \multicolumn{4}{c}{{\tt ne,hi,hr}} & \multicolumn{4}{c}{{\tt phi,phy,sci}} & \multicolumn{4}{c}{{\tt scif,so,ve}} & \multicolumn{4}{c}{{\tt ch,ex,le}} \\
			\cmidrule(lr){2-5}\cmidrule(lr){6-9}\cmidrule(lr){10-13}\cmidrule(lr){14-17}\cmidrule(lr){18-21}\cmidrule(lr){22-25}
			&{\tt ac}&{\tt an}&{\tt ba}&{\bf avg.}&{\tt ca}&{\tt cu}&{\tt do}&{\bf avg.}&{\tt ne}&{\tt hi}&{\tt hr}&{\bf avg.} &{\tt phi}&{\tt phy}&{\tt sci}&{\bf avg.}&{\tt scif}&{\tt so}&{\tt ve}&{\bf avg.}&{\tt ch}&{\tt ex}&{\tt le} & {\bf avg.} \\
			\midrule
			{Base Model} & \textcolor{gray}{.153} & \textcolor{gray}{.234} & \textcolor{gray}{.198} & \textcolor{gray}{.195} & \textcolor{gray}{.225} & \textcolor{gray}{.124} & \textcolor{gray}{.181} & \textcolor{gray}{.177} & \textcolor{gray}{.176} & \textcolor{gray}{.206} & \textcolor{gray}{.134} & \textcolor{gray}{.132}& \textcolor{gray}{.152} & \textcolor{gray}{.237} & \textcolor{gray}{.167} & \textcolor{gray}{.185} & \textcolor{gray}{.188} & \textcolor{gray}{.193} & \textcolor{gray}{.215} & \textcolor{gray}{.199} & \textcolor{gray}{.203} & \textcolor{gray}{.113} & \textcolor{gray}{.124} & \textcolor{gray}{.147} \\
			{SFT}  & .627 & .583 & .535 & .582 & .498 & .541 & .596 & .545 & .598 & .545 & .476 & .540 & .487 & .526 & .476 & .496 & .526 & .497 & .502 & .508 & .646 & .587 & \underline{.668} & .634 \\
			{PPO (tr)}  & .646 & .586 & .557 & .596 & .507 & .604 & \textcolor{gray}{.558} & .556 & .660 & .557 & .502 & .573& .538 & \textcolor{gray}{.515} & .533 & .528 & .602 & .507 & \textcolor{gray}{.478} & .529 & \underline{.654} & .604 & \textcolor{gray}{.625} & \textcolor{gray}{.628} \\
			{PPO (te)}  & \underline{.682} & \underline{.594} & .567 & \underline{.614} & \underline{.557} & .614 & \textcolor{gray}{.583} & \underline{.585} & .672 & \underline{.587} & \underline{.536} & \underline{.598} & \underline{.601} & .552 & \underline{.597} & \underline{.583} & \underline{.640} & \underline{.528} & \underline{.601} & \underline{.590} & .647 & \underline{.654} & \textcolor{gray}{.639} &\underline{.647}  \\
			{DPO}  & {.657} & \textcolor{gray}{.554} & .565 & .592 & .525 & \textcolor{gray}{.537} &  \underline{.612} & .558 & .648 & \textcolor{gray}{.506} & .527 & .560 & .507 & \underline{.582} & .557 & .549 & .580 & .527 & .554 & .554 & \textcolor{gray}{.636} & \textcolor{gray}{.452} & \textcolor{gray}{.657} & \textcolor{gray}{.582} \\
			{HPL} & \textcolor{gray}{.536} & \textcolor{gray}{.497} & \underline{.624} & \textcolor{gray}{.552} & .556 & \underline{.625} & \textcolor{gray}{.542} & .574 & \underline{.682} & \textcolor{gray}{.486} & .478 &.549& \textcolor{gray}{.342} & \textcolor{gray}{.478} & .572 & \textcolor{gray}{.467} & .582 & \textcolor{gray}{.486} & .557  & .542 & \textcolor{gray}{.486} & .627 & \textcolor{gray}{.586} & \textcolor{gray}{.566} \\
			\textbf{Ours}  & \textbf{.724} & \textbf{.607} & \textbf{.724} & \textbf{.685} & \textbf{.686} & \textbf{.687} & \textbf{.702} &\textbf{.692}& \textbf{.686} & \textbf{.652} & \textbf{.586} & \textbf{.642}& \textbf{.604} & \textbf{.686} & \textbf{.724} & \textbf{.671} & \textbf{.656} & \textbf{.567} & \textbf{.652} & \textbf{.625} & \textbf{.672} & \textbf{.680} & \textbf{.749} & \textbf{.700} \\
			\bottomrule
	\end{tabular}}
	\label{tbl:gpt4eval}
\end{table*}

Building upon prior research \citep{rafailov2024direct}, we employ GPT-4 to assess whether the model-generated answers are better than the preferred responses from the SHP dataset. Subsequently, we evaluate the win rates of all approaches against the preferred answers in the test set. This underscores the significance of fine-tuning a pretrained language model using supervised data for KAG. Compared to the {SFT} baseline, {PPO (tr)} additionally utilizes the training distributions for RL fine-tuning but struggles to enhance the results on the test distribution. This illustrates the challenge of OOD PL with RLHF. Additionally, DPO and HPL fine-tune the SFT model with PL data from the training distributions but also struggle to enhance the OOD results. The limitations are evident in the distribution bias of RM training. In contrast, we employ a meta-learning approach to learn a general RM that can be utilized for OOD policy tuning. Consequently, our method achieves superior results across all experiments.

\subsection{Effects of $\beta$ on Reward Learning}
\begin{wrapfigure}[15]{r}{0.5\textwidth}
	\vspace{-1.0cm}
	\centering
	\subfigure[Reward.]{
		\begin{minipage}[h]{0.15\textwidth}
			\centering
			\includegraphics[width=2.2cm]{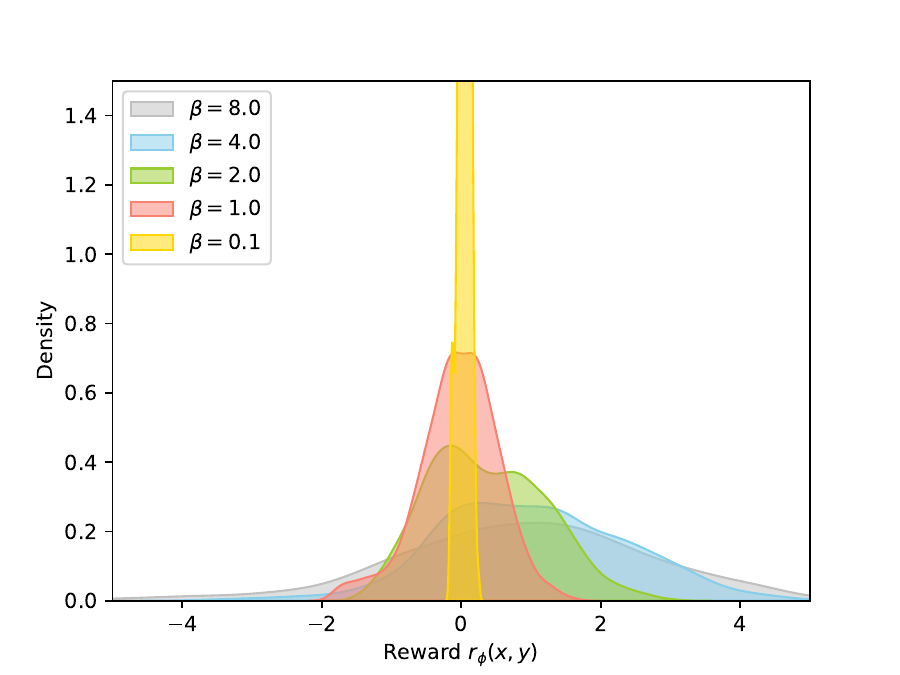}
	\end{minipage}}
	\subfigure[log. com.]{
		\begin{minipage}[h]{0.15\textwidth}
			\centering
			\includegraphics[width=2cm]{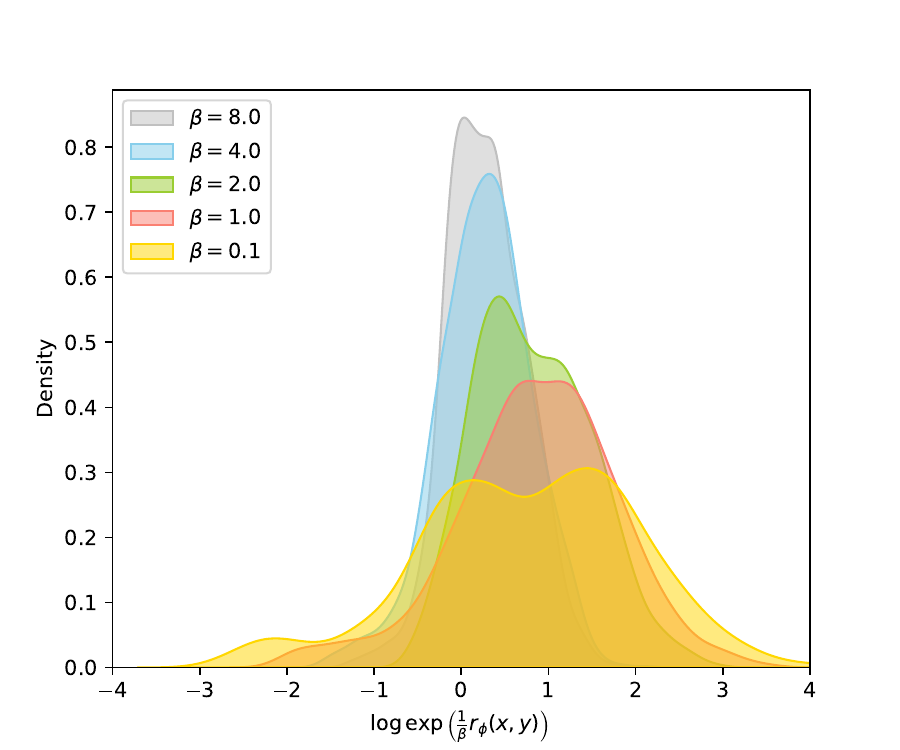}
	\end{minipage}}
	\subfigure[PL acc.]{
		\begin{minipage}[h]{0.15\textwidth}
			\centering
			\includegraphics[width=2.3cm]{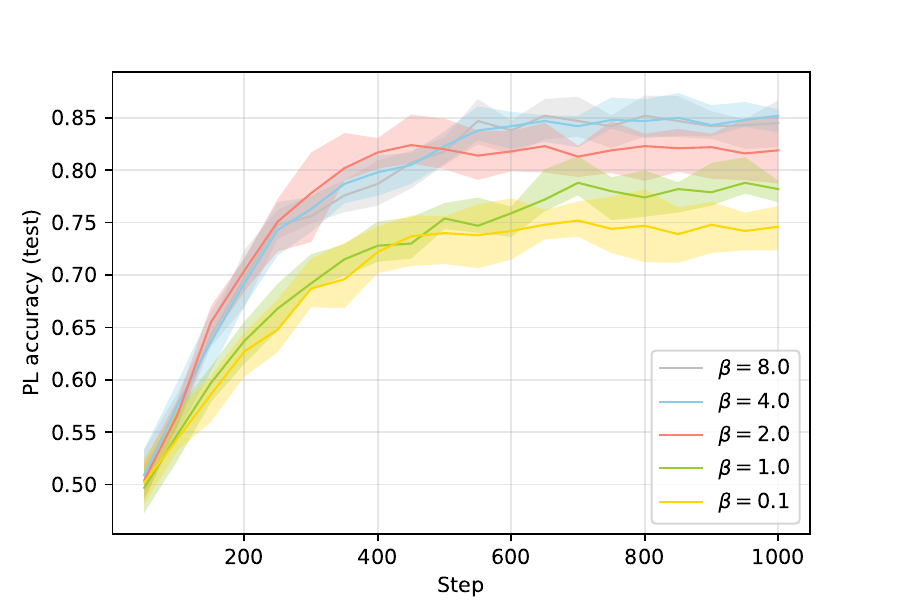}
	\end{minipage}}
	\subfigure[Reward.]{
		\begin{minipage}[h]{0.15\textwidth}
			\centering
			\includegraphics[width=2.2cm]{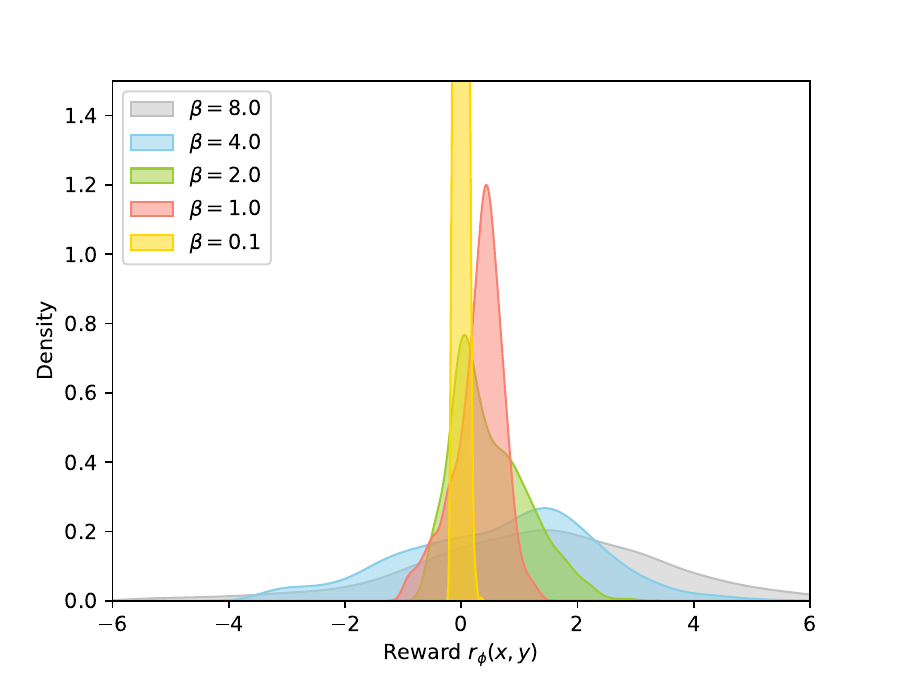}
	\end{minipage}}
	\subfigure[log. com.]{
		\begin{minipage}[h]{0.15\textwidth}
			\centering
			\includegraphics[width=2cm]{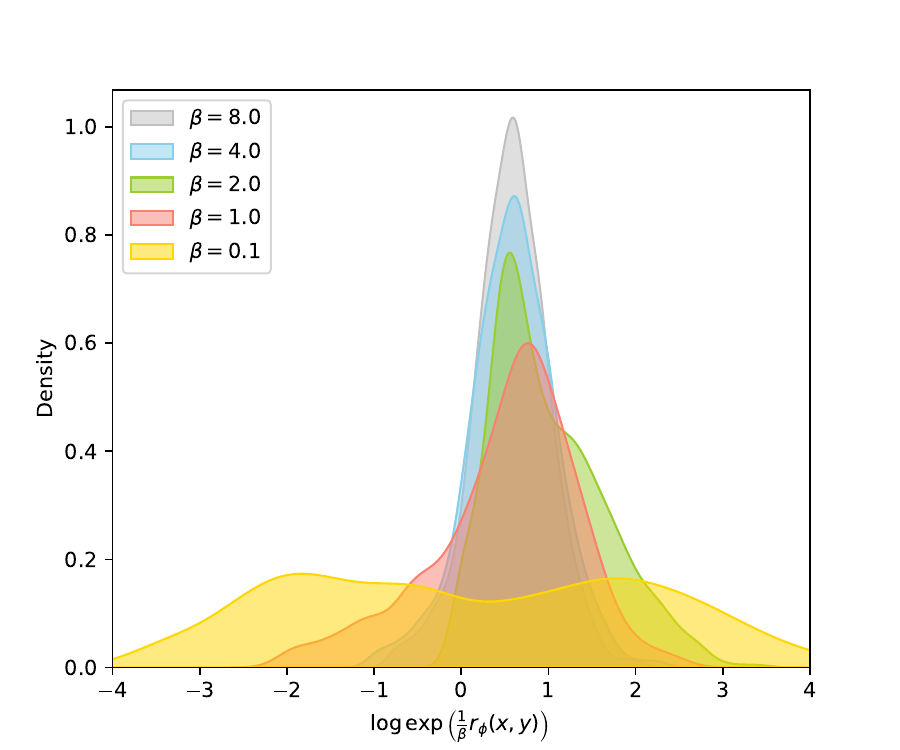}
	\end{minipage}}
	\subfigure[PL acc.]{
		\begin{minipage}[h]{0.15\textwidth}
			\centering
			\includegraphics[width=2.3cm]{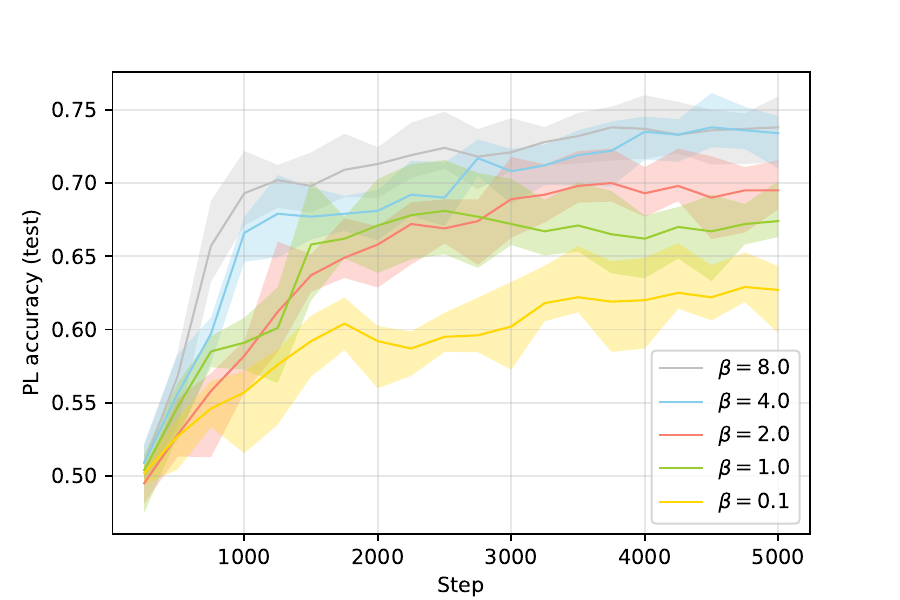}
	\end{minipage}}
	\caption{Effects of reward learning w.r.t. the reward controlling factor $\beta$ on the {\tt book} distribution (a-c) and on the {\tt ac,an,ba} distribution (d-f).} \label{fig:beta}
	\vspace{-0.5cm}
\end{wrapfigure}
Fig. \ref{fig:beta} illustrates our investigation into various values of $\beta \in {0.1,1.0,2.0,4.0,8.0}$ for reward learning. From Fig. \ref{fig:beta} (a) and (d), it is evident that when the reward controlling factor is too small, such as $\beta=0.1$, the reward distribution becomes very sharp and approaches $0$, providing limited information for policy fine-tuning. As $\beta$ increases, the reward distribution becomes flatter, which better represents the preferences of inputs for guiding policy fine-tuning. Additionally, we explore the learning of the combination term $\exp \left( \frac{1}{\beta}r_{\phi}(x,y) \right)$ for cross-entropy in the inner objective. Figures \ref{fig:beta} (b) and (e) demonstrate that as $\beta$ increases, the distribution of the combination term approaches a standard normal distribution, facilitating faster convergence of policy fine-tuning. Furthermore, we examine how $\beta$ influences PL by presenting the PL accuracy against the training steps for various values of $\beta$ in Fig. \ref{fig:beta} (c) and (f). We observe that for each value of $\beta$, as the number of training steps increases, the PL accuracy improves and eventually converges to a high plateau. Moreover, a larger value of $\beta$ yields better PL accuracy for both tasks. The experimental findings align with the theoretical result of Theorem \ref{thm:1}, suggesting that a larger number of iterations $K$ and a larger value of $\beta$ contribute to the algorithm's convergence to a more accurate stationary point.

\section{Conclusion}
Preference learning (PL) is a powerful framework for training language models to align with human values. Instead of discretely training a reward model solely for in-distribution (ID) PL, this work investigated an out-of-distribution (OOD) PL algorithm. The aim is to improve the generalization of reward modeling across various distributions. To address OOD PL, we proposed a meta-learning approach conducted through a bilevel optimization algorithm. We derived an upper bound for the convergence rate of the proposed algorithm. We conducted experiments on both controlled sentiment generation and knowledge answer generation. Results demonstrate that the generalized reward modeling achieved by our meta-learning approach effectively guides policy learning to align with human preferences.

\appendix
	\begin{footnotesize}
		\section{Parameter Estimation for Policy Learning} \label{apdx:paraest}
		Based on the objective function of policy in Eq. (\ref{eq:policyobj}), we have
		\begin{align*}
			& \mathop{\arg\max}_{\pi} \mathbb{E}_{x \sim \mathcal{D}_x} \left[ \mathbb{E}_{y\sim\pi(\cdot|x)} \left[ r_{\phi}(x,y) \right] - \beta\mathrm{D}_{\rm KL} \left( \pi(\cdot|x) || \mathcal{D}_{y|x}(\cdot) \right) \right] \\
			&= \mathop{\arg\max}_{\pi} \int_{x} \mathcal{D}_{x}(x) \left( \int_{y} \pi(y|x) r_{\phi}(x,y) dy - \beta \mathrm{D}_{\rm KL} \left( \pi(\cdot|x) || \mathcal{D}_{y|x}(\cdot) \right) \right) dx
		\end{align*}
		
		Then, we can formulate the following constrained policy optimization problem,
		\begin{align*}
			\mathop{\arg\max}_{\pi} &\int_{x} \mathcal{D}_{x}(x) \left( \int_{y} \pi(y|x) r_{\phi}(x,y) dy - \beta \mathrm{D}_{\rm KL} \left( \pi(\cdot|x) || \mathcal{D}_{y|x}(\cdot) \right) \right) dx \\
			{\rm s.t.}\ &\int_{y} \pi(y|x) dy = 1,\quad \forall x
		\end{align*}
		
		Next, we can form the Lagrangian,
		\begin{align*}
			\mathcal{L}(\pi, \alpha) =& \int_{x} \mathcal{D}_{x}(x) \left( \int_{y} \pi(y|x) r_{\phi}(x,y) dy -  \beta \mathrm{D}_{\rm KL} \left( \pi(\cdot|x) || \mathcal{D}_{y|x}(\cdot) \right) \right) dx\\
			& + \int_{x} \alpha_x \left( 1 - \int_{y} \pi(y|x) dy \right) dx,
		\end{align*}
		with $\alpha=\{\alpha_x: \forall x \in \mathcal{X} \}$ corresponding to the Lagrange multipliers.
		Now, we can differentiate $\mathcal{L}(\pi, \alpha)$ w.r.t. $\pi(y|x)$ as
		\begin{align*}
			&\frac{\partial \mathcal{L}(\pi, \alpha)}{\partial \pi(y|x)} \\
			=& \frac{\partial}{\partial \pi(y|x)} \int_{x} \int_y \pi(y|x) \left( \mathcal{D}_{x}(x) \left( r_{\phi}(x,y) - \beta \log \pi(y|x) + \beta \log \mathcal{D}_{y|x}(y)\right) - \alpha_x \right) dy + \alpha_x  dx \\
			=& \int_x \int_y \frac{\partial}{\partial \pi(y|x)} \pi(y|x) \left( \mathcal{D}_{x}(x) \left( r_{\phi}(x,y) - \beta \log \pi(y|x) + \beta \log \mathcal{D}_{y|x}(y)\right) - \alpha_x \right) dy dx \\
			=& \int_x \int_y \left( \mathcal{D}_{x}(x)r_{\phi}(x,y) - \beta\mathcal{D}_{x}(x)\log\pi(y|x) - \beta\mathcal{D}_{x}(x) + \beta\mathcal{D}_{x}(x)\log \mathcal{D}_{y|x}(y) - \alpha_x \right) dy dx
		\end{align*}
		
		Setting the above integrated term to zero, we can obtain a sufficient condition on $\frac{\partial \mathcal{L}(\pi, \alpha)}{\partial \pi(y|x)} = 0$,
		\begin{align} \label{eq:30}
			\pi_*(y|x) = \mathcal{D}_{y|x}(y) \exp\left( \frac{1}{\beta}r_{\phi}(x,y)\right) \exp\left( -\frac{\alpha_x}{\beta \mathcal{D}(x)} - 1 \right), \quad \forall y \forall x
		\end{align}
		
		Besides, the strict negativeness of the sencond-order partial derivative of $\mathcal{L}(\pi, \alpha)$ is shown by
		\begin{align*}
			\frac{\partial^2 \mathcal{L}(\pi, \alpha)}{\partial [\pi(y|x)]^2} = \int_x \int_y - \frac{\beta\mathcal{D}_{x}(x)}{\pi(y|x)} dydx < 0,
		\end{align*}
		since $\beta, \mathcal{D}_{x}(x), \pi(y|x) > 0$. This indicates the strict concaveness of $\mathcal{L}(\pi, \alpha)$. Therefore, Eq. (\ref{eq:30}) indicates a global optimization solution $\pi_* = \mathop{\arg\max}_{\pi} \mathcal{L}(\pi, \alpha)$.
		
		Since $\int_{y} \pi_*(y|x) dy = 1$, the second exponential term can be rewritten as a partition function $Z(x)$ that normalizes the conditional action distribution,
		\begin{align*}
			Z(x) = \exp\left( \frac{\alpha_x}{\beta \mathcal{D}_{x}(x)} + 1 \right) = \int_y \mathcal{D}_{y|x}(y) \exp\left( \frac{1}{\beta}r_{\phi}(x,y) \right) dy
		\end{align*}
		
		The optimal policy is therefore given by,
		\begin{align} \label{eq:normpi}
			\pi_*(y|x) = \frac{1}{Z(x)} \mathcal{D}_{y|x}(y) \exp\left( \frac{1}{\beta}r_{\phi}(x,y)\right), \quad \forall y \forall x
		\end{align}
		
		To obtain such a global optimization solution for the policy, we use a parameter estimation method, i.e., using a machine learning method to train a parameter estimated policy $\pi_{\theta}$ close to the globally optimal solution $\pi_*$ of Eq. (\ref{eq:normpi}) by solving the following supervised regression problem:
		\begin{align*}
			&\mathop{\arg\min}_{{\theta}} \mathbb{E}_{x \sim \mathcal{D}_x} \left[ \mathrm{D}_{\rm KL} \left( \pi_*(\cdot|x) || \pi_{\theta}(\cdot|x) \right) \right] \\
			=& \mathop{\arg\min}_{{\theta}} \mathbb{E}_{x \sim \mathcal{D}_x} \left[ \mathrm{D}_{\rm KL} \left( \frac{1}{Z(x)} \mathcal{D}_{y|x}(\cdot) \exp\left( \frac{1}{\beta}r_{\phi}(x,\cdot)\right) \bigg\Vert \pi_{\theta}(\cdot|x) \right) \right] \\
			=&\mathop{\arg\min}_{\theta} - \mathbb{E}_{x \sim \mathcal{D}_x} \mathbb{E}_{y \sim \mathcal{D}_{y|x}} \left[  \log \pi_{\theta}(y|x) \exp \left( \frac{1}{\beta}r_{\phi}(x,y) \right) \right]
		\end{align*}
		
		\section{Proof of Proposition \ref{pro:gradient}} \label{apdx:proproof}
		Based on the formulation of outer objective $\ell_{\rm PL}$ in Eq. (\ref{eq:outer}), we have for any $k\in\{0,1,2,\ldots,K-1\}$ and the input data $(x_k,y_k,y'_k)$,
		\begin{align*} \label{eq:pro1}
			&\frac{\partial \ell_{\rm PL}({\phi_k}, {\theta_{k,D}})}{\partial {\phi}_k} \\
			=&  - \frac{\partial}{\partial {\phi}_k} \log \sigma \left( \log \pi_{\theta_{k,D}}(y_k|x_k) - \log \pi_{\theta_{k,D}}(y'_k|x_k) \right) \\
			=& - \frac{\partial {\theta}_{k,D}}{\partial {\phi}_k} \left( \nabla_{{\theta}} \log \pi_{\theta_{k,D}}(y_k|x_k) - \nabla_{{\theta}} \log \pi_{{\theta}_{k,D}}(y'_k|x_k) \right) \\
			&\cdot \left(1 - \sigma \left( \log \pi_{\theta_{k,D}}(y_k|x_k) - \log \pi_{\theta_{k,D}}(y'_k|x_k) \right) \right)
		\end{align*}
		
		Then, let $A_k = \log \pi_{\theta_{k,D}}(y_k|x_k) -  \log \pi_{{\theta}_{k,D}}(y'_k|x_k)$, we have
		\begin{align}
			\frac{\partial \ell_{\rm PL}({\phi_k}, {\theta_{k,D}})}{\partial {\phi}_k} = - \frac{\partial {\theta}_{k,D}}{\partial {\phi}_k} \nabla_{{\theta}} A_k \left(1 - \sigma \left( A_k \right) \right)
		\end{align}
		
		Based on the formulation of inner objective $\ell_{\rm FT}$ in Eq. (\ref{eq:innerobj}), the iterative update for any $t\in \{0,1.2,\ldots,D-1\}$ based on the input data $(x_{k,y}, y_{k,t})$ takes the form 
		\begin{align*}
			{\theta}_{k,t+1} =& {\theta}_{k,t} - \alpha\nabla_{{\theta}}\ell_{\rm FT}({\phi_k}, \pi_{{k,t}}) \\
			=& {\theta}_{k,t} + \alpha\nabla_{\theta} \log \pi_{\theta_{k,t}}(y_{k,t}|x_{k,t}) \exp \left( \frac{1}{\beta}r_{\phi_k}(x_{k,t},y_{k,t}) \right)
		\end{align*}
		
		Then, we have
		\begin{align*}
			\frac{\partial {\theta}_{k,t+1}}{\partial {\phi_k}} = \frac{\partial {\theta}_{k,t}}{\partial {\phi}_k} &+ \alpha \frac{\partial {\theta}_{k,t}}{\partial {\phi}_k} \nabla^2_{{\theta}} \log \pi_{\theta_{k,t}}(y_{k,t}|x_{k,t}) \exp \left( \frac{1}{\beta}r_{\phi_k}(x_{k,t},y_{k,t}) \right)\\
			& + \frac{\alpha}{\beta} \nabla_{\phi} r_{\phi_k}(x_{k,t},y_{k,t}) \nabla_{{\theta}}^{\top} \log \pi_{\theta_{k,t}}(y_{k,t}|x_{k,t}) \exp \left( \frac{1}{\beta}r_{\phi_k}(x_{k,t},y_{k,t}) \right)
		\end{align*}
		
		To simplify the presentation, let $R_{k,t} := \exp \left( \frac{1}{\beta}r_{\phi_k}(x_{k,t},y_{k,t}) \right)$, $F_{k,t} := \log \pi_{\theta_{k,t}}(y_{k,t}|x_{k,t})$ and $H_{k,t} := r_{\phi_k}(x_{k,t},y_{k,t})$, we have
		\begin{align*}
			\frac{\partial {\theta}_{k,t+1}}{\partial {\phi}_k} =& \frac{\partial {\theta}_{k,t}}{\partial {\phi}_k} \left( I + \alpha \nabla_{{\theta}}^2 F_{k,t} R_{k,t} \right) + \frac{\alpha}{\beta} \nabla_{\phi} H_{k,t} \nabla_{\theta}^{\top} F_{k,t} R_{k,t}
		\end{align*}
		
		Telescoping the above equality over $t$ from $0$ to $D-1$ yields
		\begin{align} \label{eq:pro2}
			\begin{split}
				\frac{\partial {\theta}_{k,D}}{\partial {\phi_k}} =& \frac{\partial {\theta}_{k,0}}{\partial {\phi_k}} \prod_{t=0}^{D-1}\left( I + \alpha \nabla_{{\theta}}^2 F_{k,t} R_{k,t} \right) + \frac{\alpha}{\beta}\sum_{t=0}^{D-1} \nabla_{\phi} H_{k,t} \nabla_{{\theta}}^{\top} F_{k,t} R_{k,t} \prod_{j=t+1}^{D-1} \left( I + \alpha \nabla_{{\theta}}^2 F_{k,t} R_{k,t} \right)\\
				=& \frac{\alpha}{\beta}\sum_{t=0}^{D-1} \nabla_{\phi} H_{k,t} \nabla_{{\theta}}^{\top} F_{k,t} R_{k,t} \prod_{j=t+1}^{D-1} \left( I + \alpha \nabla_{{\theta}}^2 F_{k,t} R_{k,t} \right),
			\end{split}
		\end{align}
		
		where the second equation follows from the fact that $\frac{\partial {\theta}_{k,0}}{\partial \phi_k} = 0$. Combining Eq. (\ref{eq:pro1}) finishes the proof.
		
		\section{Supporting Lemmas}
		\begin{lemma} \label{lemma:smoothstrong}
			Based on Assumption \ref{asp:strong1}, \ref{asp:lip2} and \ref{asp:bound3}, $\ell_{\rm FT}(\phi, \theta)$ ($\mathcal{L}_{\rm FT}(\phi, \theta)$) is $\frac{BL}{\exp\beta}$-smoothness and $\frac{b\mu}{\exp\beta}$-strong convexity w.r.t. $\theta$.
		\end{lemma}
		
		\begin{proof}
			The $\frac{BL}{\exp\beta}$-smoothness of $\mathcal{L}_{\rm FT}(\phi, \theta)$ follows from
			\begin{align*}
				&\left\Vert \nabla_{\theta} \mathcal{L}_{\rm FT}(\phi, \theta) - \nabla_{\theta} \mathcal{L}_{\rm FT}(\phi, \theta') \right\Vert = \left\Vert \mathbb{E}_z\left[ \nabla_{\theta} \ell_{\rm FT}(\phi, \theta;z) \right] - \mathbb{E}_z \left[ \ell_{\rm FT}(\phi, \theta';z) \right]\right\Vert \\
				\overset{(i)}{\leq}& \mathbb{E}_z\left[ \left\Vert \nabla_{\theta} \log\pi_{\theta'}(y|x)\exp\left(\frac{1}{\beta}r_{\phi}(x,y) \right) - \nabla_{\theta} \log\pi_{\theta}(y|x)\exp\left(\frac{1}{\beta}r_{\phi}(x,y) \right) \right\Vert \right] \\
				{=}& \mathbb{E}_z\left[\exp\left(\frac{1}{\beta}r_{\phi}(x,y) \right) \left\Vert  \nabla_{\theta} \log\pi_{\theta'}(y|x) -  \nabla_{\theta} \log\pi_{\theta}(y|x)  \right\Vert\right] \\
				\overset{(ii)}{\leq}& \frac{BL}{\exp\beta} \Vert \theta - \theta' \Vert,
			\end{align*}
			where $(i)$ follows from Jensen's inequality and the definition of $\mathcal{L}_{\rm FT}(\phi, \theta)$ $(ii)$ follows from the $L$-smoothness of $-\log\pi_\theta(y|x)$ in Assumption \ref{asp:lip2} and the $B$-boundness of $\exp r_\phi(x,y)$ in Asssumption \ref{asp:bound3}.
			
			The $\frac{b\mu}{\exp\beta}$-strong convexity of $\mathcal{L}_{\rm FT}(\phi, \theta)$ follows from
			\begin{align*}
				&\left\langle \nabla_{\theta}\mathcal{L}_{\rm FT}(\phi, \theta) - \nabla_{\theta}\mathcal{L}_{\rm FT}(\phi, \theta'), \theta-\theta'  \right\rangle = \mathbb{E}_z \left[  \left\langle \nabla_{\theta}\ell_{\rm FT}(\phi, \theta) - \nabla_{\theta}\ell_{\rm FT}(\phi, \theta'), \theta-\theta'  \right\rangle \right] \\
				\overset{(i)}{=}& \mathbb{E}_z \left[ \left\langle  \nabla_{\theta}\log\pi_{\theta'}(y|x) \exp\left(\frac{1}{\beta}r_{\phi}(x,y)\right) -  \nabla_{\theta}\log\pi_{\theta}(y|x) \exp\left(\frac{1}{\beta}r_{\phi}(x,y)\right), \theta-\theta'  \right\rangle \right] \\
				\overset{(ii)}{\geq}& \mathbb{E}_z \left[ \exp\left(\frac{1}{\beta}r_{\phi}(x,y)\right) \left\Vert \theta - \theta' \right\Vert^2 \right] \overset{(iii)}{\geq} \frac{b\mu}{\exp\beta} \Vert \theta - \theta' \Vert^2,
			\end{align*} 
			where $(i)$ follows from the definition of $\mathcal{L}_{\rm FT}(\phi, \theta)$, $(ii)$ follows from the $\mu$-strong convexity of $-\log\pi_{\theta}(y|x)$ in Assumption \ref{asp:strong1}, $(iii)$ follows from the boundness assumption of $\exp(r_{\theta}(x,y))$ in Assumption \ref{asp:bound3}. 
		\end{proof}
		
		\begin{lemma} \label{lemma:lip}
			Based on Assumption \ref{asp:lip2} and \ref{asp:bound3}, we have the Lipschitz properties of $\nabla_{\phi}\nabla_{\theta} \ell_{\rm FT}(\phi, \theta)$ and $\nabla_{\theta}^2 \ell_{\rm FT}(\phi, \theta)$,
			\begin{align}
				\left\Vert \nabla_{\phi}\nabla_{\theta} \ell_{\rm FT}(\phi, \theta) - \nabla_{\phi}\nabla_{\theta} \ell_{\rm FT}(\phi, \theta') \right\Vert &\leq \frac{BTL}{\beta\exp\beta} \Vert \theta - \theta' \Vert \\
				\left\Vert \nabla_{\theta}^2 \ell_{\rm FT}(\phi, \theta) - \nabla_{\theta}^2 \ell_{\rm FT}(\phi, \theta') \right\Vert  &\leq  \frac{B\rho}{\exp\beta}\Vert \theta - \theta' \Vert \\
				\left\Vert \nabla_{\phi}\nabla_{\theta} \ell_{\rm FT}(\phi, \theta) - \nabla_{\phi}\nabla_{\theta} \ell_{\rm FT}(\phi', \theta) \right\Vert &\leq \frac{MB(P+T^2)}{\beta\exp\beta} \Vert \phi-\phi' \Vert \\
				\left\Vert \nabla_{\theta}^2 \ell_{\rm FT}(\phi, \theta) - \nabla_{\theta}^2 \ell_{\rm FT}(\phi', \theta) \right\Vert &\leq \frac{BTL}{\beta\exp\beta} \Vert \phi-\phi' \Vert
			\end{align}
		\end{lemma}
         
		\begin{proof}
			Based on the definition of $\ell_{\rm FT}(\phi, \theta)$, we have for any $(x,y)$ and any $\theta, \theta'$,
			\begin{align*}
				&\left\Vert \nabla_{\phi}\nabla_{\theta} \ell_{\rm FT}(\phi, \theta) - \nabla_{\phi}\nabla_{\theta} \ell_{\rm FT}(\phi, \theta') \right\Vert\\
				=& \left\Vert \frac{1}{\beta} \exp\left(\frac{r_\phi(x,y)}{\beta}\right) \nabla_{\phi}^{\top} r_{\phi}(x,y) \nabla_{\theta}\log\pi_{\theta'}(y|x) - \frac{1}{\beta} \exp\left(\frac{r_\phi(x,y)}{\beta}\right) \nabla_{\phi}^{\top} r_{\phi}(x,y) \nabla_{\theta}\log\pi_{\theta}(y|x) \right\Vert \\
				{=}& \left\Vert \frac{1}{\beta} \exp\left(\frac{r_\phi(x,y)}{\beta}\right) \nabla_{\phi}^{\top} r_{\phi}(x,y) \right\Vert \left\Vert  \nabla_{\theta}\log\pi_{\theta'}(y|x) - \nabla_{\theta}\log\pi_{\theta}(y|x) \right\Vert \\
				\overset{(i)}{\leq}& \frac{BTL}{\beta\exp\beta} \left\Vert \theta - \theta' \right\Vert,
			\end{align*}
			where $(i)$ follows from Assumption \ref{asp:lip2} and \ref{asp:bound3}.
			
			Similarly, we also have
			\begin{align*}
				&\left\Vert \nabla_{\theta}^2 \ell_{\rm FT}(\phi, \theta) - \nabla_{\theta}^2 \ell_{\rm FT}(\phi, \theta') \right\Vert  \\
				=& \left\Vert \nabla_{\theta}^2 \log\pi_{\theta'}(y|x)\exp\left( \frac{r_\phi(x,y)}{\beta} \right) - \nabla_{\theta}^2 \log\pi_{\theta}(y|x)\exp\left( \frac{r_\phi(x,y)}{\beta} \right) \right\Vert \\
				\leq& \exp\left( \frac{r_\phi(x,y)}{\beta} \right) \left\Vert \nabla_{\theta}^2 \log\pi_{\theta'}(y|x) - \nabla_{\theta}^2 \log\pi_{\theta}(y|x) \right\Vert \\
				\overset{(i)}{\leq}& \frac{B\rho}{\exp\beta}\Vert \theta - \theta' \Vert,
			\end{align*}
			where $(i)$ follows from Assumption \ref{asp:lip2} and \ref{asp:bound3}.
			
		    Similarly, we have
			\begin{align*}
				&\left\Vert \nabla_{\phi}\nabla_{\theta} \ell_{\rm FT}(\phi, \theta) - \nabla_{\phi}\nabla_{\theta} \ell_{\rm FT}(\phi', \theta) \right\Vert\\
                =& \left\Vert \frac{1}{\beta} \exp\left(\frac{r_{\phi'}(x,y)}{\beta}\right) \nabla_{\phi}^{\top} r_{\phi'}(x,y) \nabla_{\theta}\log\pi_{\theta}(y|x) - \frac{1}{\beta} \exp\left(\frac{r_\phi(x,y)}{\beta}\right) \nabla_{\phi}^{\top} r_{\phi}(x,y) \nabla_{\theta}\log\pi_{\theta}(y|x) \right\Vert \\
                \leq& \frac{1}{\beta}\left\Vert \nabla_{\theta}\log\pi_{\theta}(y|x) \right\Vert \left\Vert \exp\left(\frac{r_{\phi'}(x,y)}{\beta}\right) \nabla_{\phi}^{\top} r_{\phi'}(x,y) - \exp\left(\frac{r_\phi(x,y)}{\beta}\right) \nabla_{\phi}^{\top} r_{\phi}(x,y) \right\Vert \\
                \leq& \frac{MB}{\beta\exp\beta} \left\Vert \nabla_{\phi}^{\top} r_{\phi'}(x,y) - \nabla_{\phi}^{\top} r_{\phi}(x,y) \right\Vert + \frac{MT}{\beta} \left\Vert \exp\left(\frac{r_{\phi'}(x,y)}{\beta}\right) - \exp\left(\frac{r_{\phi}(x,y)}{\beta}\right) \right\Vert \\
                \leq& \frac{MB(P+T^2)}{\beta\exp\beta} \Vert \phi-\phi' \Vert
			\end{align*}
		
		    Similarly, we have
			\begin{align*}
				&\left\Vert \nabla_{\theta}^2 \ell_{\rm FT}(\phi, \theta) - \nabla_{\theta}^2 \ell_{\rm FT}(\phi', \theta) \right\Vert  \\
                =& \left\Vert \nabla_{\theta}^2 \log\pi_{\theta}(y|x)\exp\left( \frac{r_{\phi'}(x,y)}{\beta} \right) - \nabla_{\theta}^2 \log\pi_{\theta}(y|x)\exp\left( \frac{r_\phi(x,y)}{\beta} \right) \right\Vert \\
                \leq& \Vert \nabla_{\theta}^2 \log\pi_{\theta}(y|x) \Vert \left\Vert \exp\left(\frac{r_{\phi'}(x,y)}{\beta}\right) - \exp\left(\frac{r_{\phi}(x,y)}{\beta}\right) \right\Vert \\
                \leq& \frac{BTL}{\beta\exp\beta} \Vert \phi-\phi' \Vert
		    \end{align*}
		\end{proof}
        
		\begin{lemma} \label{lemma3}
			Based on Assumption \ref{asp:lip2}, we have the Lipschitz property of the preference learning loss $\ell_{\rm PL}$, $\nabla_{\theta}\ell_{\rm  PL}(\phi, \theta)$,
			\begin{align}
				\Vert \ell_{\rm PL}(\phi, \theta) - \ell_{\rm PL}(\phi, \theta') \Vert \leq& 2M \Vert \theta - \theta' \Vert  \\ \label{eq:losspllip1}
				\left\Vert \nabla_{\theta}\ell_{\rm  PL}(\phi, \theta) - \nabla_{\theta}\ell_{\rm  PL}(\phi, \theta') \right\Vert \leq& (2L + M^2)\Vert \theta - \theta' \Vert 
			\end{align}\label{eq:losspllip2}
		\end{lemma}
         
		\begin{proof}
			For Eq. (\ref{eq:losspllip1}), we have
			\begin{align*}
				&\Vert \ell_{\rm PL}(\phi, \theta) - \ell_{\rm PL}(\phi, \theta') \Vert \\
				=& \Vert \log \sigma \left( \log\pi_{\theta}(y|x) - \log\pi_{\theta}(y'|x) \right) - \log \sigma \left( \log\pi_{\theta'}(y|x) - \log\pi_{\theta'}(y'|x) \right) \Vert \\
				\overset{(i)}{\leq} & \Vert \log\pi_{\theta}(y|x) - \log\pi_{\theta}(y'|x) - \log\pi_{\theta'}(y|x) + \log\pi_{\theta'}(y'|x)  \Vert \\
				\overset{(ii)}{\leq} & \Vert \log\pi_{\theta}(y|x) - \log\pi_{\theta'}(y|x)\Vert + \Vert \log\pi_{\theta'}(y'|x) -  \log\pi_{\theta}(y'|x) \Vert \\
				\overset{(iii)}{\leq} & 2M \Vert \theta - \theta' \Vert,
			\end{align*}
			where $(i)$ follows from the $1$-Lipschitz of $\log\sigma(\cdot)$, $(ii)$ follows from the triangle inequality and $(iii)$ follows from Assumption \ref{asp:lip2}.
			
			For Eq. (\ref{eq:losspllip2}), based on the triangle inequality, we have
			\begin{align*}
				&\left\Vert \nabla_{\theta}\ell_{\rm  PL}(\phi, \theta) - \nabla_{\theta}\ell_{\rm  PL}(\phi, \theta') \right\Vert \\
				=& \Vert \sigma( \log\pi_{\theta}(y'|x) - \log\pi_{\theta}(y|x) ) \left( \nabla_{\theta} \log \pi_{\theta}(y|x) -  \nabla_{\theta} \log \pi_{\theta}(y'|x) \right) \\
				& \quad - \sigma( \log\pi_{\theta'}(y'|x) - \log\pi_{\theta'}(y|x) ) \left( \nabla_{\theta} \log \pi_{\theta'}(y|x) -  \nabla_{\theta} \log \pi_{\theta'}(y'|x) \right) \Vert \\
				\leq& \Vert  \sigma( \log\pi_{\theta}(y'|x) - \log\pi_{\theta}(y|x) ) \left( \nabla_{\theta} \log \pi_{\theta}(y|x) -  \nabla_{\theta} \log \pi_{\theta}(y'|x) - \nabla_{\theta} \log \pi_{\theta'}(y|x)  +  \nabla_{\theta} \log \pi_{\theta'}(y'|x) \right) \Vert \\
				& + \Vert \left(\sigma( \log\pi_{\theta}(y'|x) - \log\pi_{\theta}(y|x) ) - \sigma( \log\pi_{\theta'}(y'|x) - \log\pi_{\theta'}(y|x) )\right) \left( \nabla_{\theta} \log \pi_{\theta'}(y|x) -  \nabla_{\theta} \log \pi_{\theta'}(y'|x) \right)  \Vert \\
				\overset{(i)}{\leq}& \Vert  \nabla_{\theta} \log \pi_{\theta}(y|x) - \nabla_{\theta} \log \pi_{\theta'}(y|x) \Vert + \Vert \nabla_{\theta} \log \pi_{\theta}(y'|x) -  \nabla_{\theta} \log \pi_{\theta'}(y'|x) \Vert \\
				& + \frac{M}{2} \left(\Vert \log\pi_{\theta}(y'|x) - \log\pi_{\theta'}(y'|x) \Vert + \Vert \log\pi_{\theta}(y|x) - \log\pi_{\theta'}(y|x) \Vert \right) \\
				\overset{(ii)}{\leq}& (2L + M^2)\Vert \theta - \theta' \Vert,
			\end{align*}
			where the first term of $(i)$ follows from the fact w.r.t. logistic funtion that $\sigma(\cdot) < 1$, the second term of $(i)$ follows from the $\frac{1}{4}$-lipschitz of logistic funtion $\sigma$ and Assumption \ref{asp:lip2}, $(ii)$ follows from Assumption \ref{asp:lip2}.
		\end{proof}
		
		\begin{lemma}\label{lemma4}
			Based on Assumption \ref{asp:lip2}, for any $k \in \{ 0,1,2,\ldots, K-1 \}$, we have
			\begin{align}
				\left\Vert \frac{\partial {\theta_*}}{\partial \phi_k}  \right\Vert \leq \frac{MT}{\beta \mu}
			\end{align}
		\end{lemma}
		
		\begin{proof}
			Based on the optimality of ${\theta_*}$, we have $\nabla_{{\theta}} \mathcal{L}_{\rm FT}({\phi_k}, {\theta_*}) = 0$, which in conjunction the implicit differentiation theorem, yields
			\begin{align} \label{eq:2ordopt}
				\nabla_{\phi}\nabla_{{\theta}} \mathcal{L}_{\rm FT}({\phi_k}, {\theta_*}) + \frac{\partial {\theta_*}}{\partial \phi_k}\nabla_{{\theta}}^2 \mathcal{L}_{\rm FT}({\phi_k}, {\theta_*}) = 0,
			\end{align}
			then, we have 
			\begin{align*}
				&\left\Vert \frac{\partial {\theta_*}}{\partial \phi_k}  \right\Vert\\
				=& \left\Vert \nabla_{\phi}\nabla_{{\theta}} \mathcal{L}_{\rm FT}({\phi_k}, {\theta_*})\left[\nabla_{{\theta}}^2 \mathcal{L}_{\rm FT}({\phi_k}, {\theta_*})\right]^{-1} \right\Vert \\
				{=}& \left\Vert  \mathbb{E} \left[\frac{1}{\beta} \nabla_{\phi}^{\top} r_{\phi_k}(x,y) \nabla_{{\theta}} \log\pi_{\theta_*}(y|x) \exp\left(\frac{1}{\beta}r_{\phi_{k}}(x,y)\right)\right] \mathbb{E}\left[ \exp\left(\frac{1}{\beta}r_{\phi_{k}}(x,y)\right)\nabla_{{\theta}}^2 \log\pi_{\theta_*}(y|x) \right]^{-1} \right\Vert \\
				\overset{(i)}{\leq}& \frac{1}{\beta}\mathbb{E} \left[\left\Vert \nabla_{\phi}^{\top} r_{\phi_k}(x,y) \right\Vert \left\Vert  \nabla_{{\theta}} \log\pi_{\theta_*}(y|x) \right\Vert \right] \mathbb{E} \left[\left\Vert \nabla^2_{{\theta}} \log\pi_{\theta_*}(y|x) \right\Vert^{-1}\right] \overset{(ii)}{\leq} \frac{MT}{\beta \mu},
			\end{align*}
			where $(i)$ follows from Jensen's inequality, $(ii)$ follows from Assumption \ref{asp:strong1}, \ref{asp:lip2}.
		\end{proof}

		\section{Proof of the Main Theorem} \label{apdx:proofthm1}
		\subsection{Useful Lemmas}
		\begin{lemma} \label{lemma:smooth-phi}
			Based on Assumption \ref{asp:strong1}, \ref{asp:lip2} and \ref{asp:bound3}, $\ell_{\rm PL}(\phi,\theta)$ ($\mathcal{L}_{\rm PL}(\phi,\theta)$) is $P_{\phi}$-smoothness w.r.t. $\phi$, where
			\begin{align}
				P_{\phi} = \frac{4M^3B(P+T^2)}{\beta b \mu} + \frac{(6L+M^2)B^2T^2M^2}{\beta^2 b^2 \mu^2} + \frac{2\rho B^3 T^2 M^3}{\beta^2 b^3  \mu^3}
			\end{align}
		\end{lemma}
		
		\begin{proof}
			The smoothness of $\ell_{\rm PL}(\phi,\theta)$ is adopted from Lemma 2.2 in \cite{ghadimi2018approximation} with Lemma \ref{lemma:smoothstrong}, \ref{lemma:lip} and \ref{lemma3}. Then, the smoothness of $\mathcal{L}_{\rm PL}(\phi,\theta)$ follows from Jensen's inequality.
		\end{proof}

		\begin{lemma} \label{lemma5}
			Based on Assumption \ref{asp:strong1}, \ref{asp:lip2} and \ref{asp:bound3}, we have for any $k \in \{ 0,1,\ldots, K-1 \}$ and any $t \in \{ 0,1,\ldots,D-1 \}$, 
			\begin{align} \label{eq:lemma7}
				\mathbb{E} \left[ \left\Vert \theta_{k,t+1} - \theta_* \right\Vert^2 \right] \leq \left( 1 -  \frac{2\alpha b\mu BL}{(b\mu + BL)\exp\beta} \right) \mathbb{E}\left[ \left\Vert \theta_{k,t} - \theta_* \right\Vert \right] + \alpha^2\sigma^2
			\end{align}
		\end{lemma}
		
		\begin{proof}
			Based on the updating rule of $\theta_{k,t+1} = \theta_{k,t} - \alpha \nabla_{\theta} \ell_{\rm FT} \left( {\phi_{k}}, \theta_{k,t}; z_{k,t} \right)$ and algebraic manipulations, we have
			\begin{align*} 
				\left\Vert \theta_{k,t+1} - \theta_* \right\Vert^2 =& \left\Vert \theta_{k,t+1} - \theta_{k,t} \right\Vert^2 + 2\langle \theta_{k,t+1} - \theta_{k,t}, \theta_{k,t} - \theta_* \rangle + \left\Vert \theta_{k,t} - \theta_* \right\Vert^2 \\
				=& \alpha^2 \left\Vert \nabla_{\theta} \ell_{\rm FT} \left( {\phi_{k}}, \theta_{k,t}; z_{k,t} \right) \right\Vert^2 + 2\alpha \langle \nabla_{\theta} \ell_{\rm FT} \left( {\phi_{k}}, \theta_{k,t}; z_{k,t} \right), \theta_{k,t} - \theta_* \rangle + \left\Vert \theta_{k,t} - \theta_* \right\Vert^2
			\end{align*}
			
			Taking expectation on both sides conditioning on $\theta_{k,t}$, we have
			\begin{align*}
				&\mathbb{E} \left[ \left\Vert \theta_{k,t+1} - \theta_* \right\Vert^2 \right] \\
				\overset{(i)}{\leq}& \alpha^2\left( \sigma^2 + \left\Vert \nabla_{\theta} \mathcal{L}_{\rm FT} \left( {\phi_{k}}, \theta_{k,t} \right) \right\Vert^2 \right) - 2\alpha\langle \nabla_{\theta} \mathcal{L}_{\rm FT} \left( {\phi_{k}}, \theta_{k,t} \right), \theta_{k,t} - \theta_* \rangle + \left\Vert \theta_{k,t} - \theta_* \right\Vert^2 \\
				\overset{(ii)}{\leq}& \alpha^2 \sigma^2 + \alpha^2\left\Vert \nabla_{\theta} \mathcal{L}_{\rm FT} \left( {\phi_{k}}, \theta_{k,t} \right) \right\Vert^2 \\
				&- 2\alpha \left( \frac{b\mu BL}{(b\mu + BL)\exp\beta}\left\Vert \theta_{k,t} - \theta_* \right\Vert^2 + \frac{\exp\beta \left\Vert \nabla_{\theta} \mathcal{L}_{\rm FT} \left( {\phi_{k}}, \theta_{k,t} \right) \right\Vert^2}{b\mu + BL}\right) + \left\Vert \theta_{k,t} - \theta_* \right\Vert^2 \\
				=& \alpha^2 \sigma^2 - \alpha \left( \frac{2\exp\beta}{b\mu + BL} - \alpha \right) \left\Vert \nabla_{\theta} \mathcal{L}_{\rm FT} \left( {\phi_{k}}, \theta_{k,t} \right) \right\Vert^2 + \left( 1 -  \frac{2\alpha b\mu BL}{(b\mu + BL)\exp\beta} \right) \left\Vert \theta_{k,t} - \theta_* \right\Vert^2,
			\end{align*}
			where $(i)$ follows from the definition of the variance ${\rm Var[X]} = \mathbb{E}[X^2] - \mathbb{E}[X]^2$, $(ii)$ follows from the fact that the $\frac{b\mu}{\exp\beta}$-strong convexity and $\frac{BL}{\exp\beta}$-smoothness of $\mathcal{L}_{\rm FT}({\phi}, {\theta})$ w.r.t. $\theta$ in Lemma \ref{lemma:smoothstrong} deduces that $\langle \nabla_{\theta} \mathcal{L}_{\rm FT} \left( {\phi_{k}}, \theta_{k,t} \right) - \nabla_{\theta} \mathcal{L}_{\rm FT} \left( {\phi_{k}}, \theta_* \right), \theta_{k,t} - \theta_* \rangle \geq \frac{b\mu BL}{(b\mu + BL)\exp\beta} \Vert \theta_{k,t} - \theta_* \Vert^2 + \frac{\exp\beta}{b\mu + BL}\Vert \nabla_{\theta} \mathcal{L}_{\rm FT} \left( {\phi_{k}}, \theta_{k,t} \right) \Vert^2$ (Theorem 2.1.12 in \citep{nesterov2018lectures}), where $\nabla_{\theta} \mathcal{L}_{\rm FT} \left( {\phi_{k}}, \theta_* \right) = 0$. 
			
			Using the condition $\alpha \leq \frac{2\exp\beta}{b\mu + BL}$ and unconditioning on $\theta_{k,t-1}$, we have Eq. (\ref{eq:lemma7}).
		\end{proof}
		
		\begin{lemma} \label{lemma:4}
			Based on Assumption \ref{asp:strong1}, \ref{asp:lip2} and \ref{asp:bound3}, we have for any $k \in \{0,1,2,\ldots, K-1\}$,
			\begin{align}\label{eq:lemma4}
				\begin{split}
					&\mathbb{E} \left[\left\Vert \frac{\partial \theta_{k,D}}{\partial \phi_k} - \frac{\partial \theta_*}{\partial \phi_k} \right\Vert^2\right] + \frac{T^2(2L+M^2)^2}{4\beta^2\mu^2} \mathbb{E}\left[ \left\Vert \theta_{k,D} - \theta_* \right\Vert^2\right] \\
					\leq& \left(3\left(1-\frac{\alpha b\mu}{\exp\beta}\right)^{2}\right)^D \frac{T^2(2L+M^2)^2}{4\beta^2\mu^2} \left(\frac{4M^2}{(2L+M^2)^2}  +  \left\Vert \theta_{0} - \theta_* \right\Vert^2\right) \\
					& + \left( \frac{\alpha^2\sigma^2T^2(2L+M^2)^2}{4\beta^2\mu^2} +  6\alpha^2 \left( \sigma_1^2 + \frac{M^2 T^2}{\beta^2 \mu^2}\sigma_2^2 \right) \right)\sum_{t=0}^{D-1} \left(3\left(1-\frac{\alpha b\mu}{\exp\beta}\right)^{2}\right)^t
				\end{split}
			\end{align}
		\end{lemma}
		
		\begin{proof}
			Based on the updating rule of SGD $\theta_{k,t+1} = \theta_{k,t} - \alpha \nabla_{{\theta}} \ell_{\rm FT} ({\phi_k}, {\theta_{k,t}}; z_{k,t})$, for any $t \in \{0,1,2,\ldots,D-1 \}$, we have
			\begin{align}\label{eq:sgdft}
				\frac{\partial \theta_{k,t+1}}{\partial \phi_k} = \frac{\partial \theta_{k,t}}{\partial \phi_k} - \alpha \left( \nabla_{\phi} \nabla_{{\theta}} \ell_{\rm FT} ({\phi_k}, {\theta_{k,t}}; z_{k,t}) - \frac{\partial \theta_{k,t}}{\partial \phi_k} \nabla^2_{{\theta}} \ell_{\rm FT} ({\phi_k}, {\theta_{k,t}}; z_{k,t}) \right)
			\end{align}
			
			Based on the optimality of ${\theta_*}$, we have $\nabla_{{\theta}} \mathcal{L}_{\rm FT}({\phi_k}, {\theta_*}) = 0$, which in conjunction the implicit differentiation theorem, yields
			\begin{align} \label{eq:2ordopt}
				\nabla_{\phi}\nabla_{{\theta}} \mathcal{L}_{\rm FT}({\phi_k}, {\theta_*}) + \frac{\partial {\theta_*}}{\partial \phi_k}\nabla_{{\theta}}^2 \mathcal{L}_{\rm FT}({\phi_k}, {\theta_*}) = 0
			\end{align}
			
			Based on Eq. (\ref{eq:sgdft}) and Eq. (\ref{eq:2ordopt}), we have
			\begin{align*}
				\frac{\partial \theta_{k,t+1}}{\partial \phi_k} - \frac{\partial \theta_*}{\partial \phi_k} =& \frac{\partial \theta_{k,t}}{\partial \phi_k} - \frac{\partial \theta_*}{\partial \phi_k} - \alpha \left( \nabla_{\phi} \nabla_{{\theta}} \ell_{\rm FT} ({\phi_k}, {\theta_{k,t}}; z_{k,t}) - \frac{\partial \theta_{k,t}}{\partial \phi_k} \nabla^2_{{\theta}} \ell_{\rm FT} ({\phi_k}, {\theta_{k,t}}; z_{k,t}) \right) \\
				&+ \alpha \left( \nabla_{\phi}\nabla_{{\theta}} \mathcal{L}_{\rm FT}({\phi_k}, {\theta_*}) + \frac{\partial {\theta_*}}{\partial \phi_k}\nabla_{{\theta}}^2 \mathcal{L}_{\rm FT}({\phi_k}, {\theta_*}) \right) \\
				=& \frac{\partial \theta_{k,t}}{\partial \phi_k} - \frac{\partial \theta_*}{\partial \phi_k} - \alpha \left( \nabla_{\phi} \nabla_{{\theta}} \ell_{\rm FT} ({\phi_k}, {\theta_{k,t}}; z_{k,t}) - \nabla_{\phi}\nabla_{{\theta}} \mathcal{L}_{\rm FT}({\phi_k}, {\theta_*}) \right) \\
				&- \alpha \left(\frac{\partial \theta_{k,t}}{\partial \phi_k} - \frac{\partial \theta_*}{\partial \phi_k} \right) \nabla^2_{{\theta}} \ell_{\rm FT} ({\phi_k}, {\theta_{k,t}}; z_{k,t}) \\
				&+ \alpha \frac{\partial \theta_*}{\partial \phi_k} \left( \nabla_{{\theta}}^2 \mathcal{L}_{\rm FT}({\phi_k}, {\theta_*}) - \nabla^2_{{\theta}} \ell_{\rm FT} ({\phi_k}, {\theta_{k,t}}; z_{k,t}) \right)
			\end{align*}
			
			Thus, taking the $\ell_2$-norm and expectation, we have
			\begin{align*}
				&\mathbb{E} \left[\left\Vert \frac{\partial \theta_{k,t+1}}{\partial \phi_k} - \frac{\partial \theta_*}{\partial \phi_k} \right\Vert^2\right] \\
				\leq& 3\mathbb{E}\left[\left\Vert I - \alpha \nabla^2_{{\theta}} \ell_{\rm FT} ({\phi_k}, {\theta_{k,t}}; z_{k,t}) \right\Vert^2 \left\Vert \frac{\partial \theta_{k,t}}{\partial \phi_k} - \frac{\partial \theta_*}{\partial \phi_k} \right\Vert^2 \right]\\
				&+ 3\alpha^2 \mathbb{E} \left[ \left\Vert \nabla_{\phi} \nabla_{{\theta}} \ell_{\rm FT} ({\phi_k}, {\theta_{k,t}}; z_{k,t}) - \nabla_{\phi}\nabla_{{\theta}} \mathcal{L}_{\rm FT}({\phi_k}, {\theta_*}) \right\Vert^2 \right]\\
				&+ 3\alpha^2 \mathbb{E} \left[ \left\Vert \frac{\partial \theta_*}{\partial \phi_k}\right\Vert^2 \left\Vert \nabla_{{\theta}}^2 \mathcal{L}_{\rm FT}({\phi_k}, {\theta_*}) - \nabla^2_{{\theta}} \ell_{\rm FT} ({\phi_k}, {\theta_{k,t}}; z_{k,t}) \right\Vert^2 \right] \\
				\overset{(i)}{\leq}& 3\left(1-\frac{\alpha b\mu}{\exp\beta}\right)^2 \mathbb{E}\left[\left\Vert \frac{\partial \theta_{k,t}}{\partial \phi_k} - \frac{\partial \theta_*}{\partial \phi_k} \right\Vert^2\right] + \frac{6\alpha^2B^2T^2}{\beta^2\exp2\beta}\left(L^2 +\frac{M^2\rho^2}{\mu^2}\right) \mathbb{E}\left[ \left\Vert \theta_{k,t} - \theta_* \right\Vert^2\right] \\
				&+ 6\alpha^2 \left( \sigma_1^2 + \frac{M^2 T^2}{\beta^2 \mu^2}\sigma_2^2 \right),
			\end{align*}
			
			where $(i)$ uses Lemma \ref{lemma:smoothstrong} and the fact that
			\begin{align*}
				&\mathbb{E}\left[ \left\Vert \nabla_{\phi} \nabla_{{\theta}} \ell_{\rm FT} ({\phi_k}, {\theta_{k,t}}; z) - \nabla_{\phi}\nabla_{{\theta}} \mathcal{L}_{\rm FT}({\phi_k}, {\theta_*}) \right\Vert^2 \right] \\
				\leq& 2\mathbb{E}\left[ \left\Vert \nabla_{\phi} \nabla_{{\theta}} \ell_{\rm FT} ({\phi_k}, {\theta_{k,t}}; z) - \nabla_{\phi}\nabla_{{\theta}} \ell_{\rm FT}({\phi_k},  {\theta_*};z) \right\Vert^2 \right] \\
				&+ 2\mathbb{E}\left[\left\Vert \nabla_{\phi}\nabla_{{\theta}} \ell_{\rm FT}({\phi_k},  {\theta_*};z) - \nabla_{\phi}\nabla_{{\theta}} \mathcal{L}_{\rm FT}({\phi_k}, {\theta_*}) \right\Vert^2 \right] \\
				\overset{(i)}{\leq}& 2\frac{B^2T^2L^2}{\beta^2\exp2\beta}\mathbb{E}\left[ \left\Vert \theta_{k,t} - \theta_* \right\Vert^2 \right] + 2\sigma_1^2;
			\end{align*}
			\begin{align*}
				& \mathbb{E}\left[\left\Vert \nabla_{{\theta}}^2 \mathcal{L}_{\rm FT}({\phi_k}, {\theta_*}) - \nabla^2_{{\theta}} \ell_{\rm FT} ({\phi_k}, {\theta_{k,t}}; z_{k,t}) \right\Vert^2\right] \\
				\leq& 2\mathbb{E}\left[ \left\Vert \nabla^2_{{\theta}} \ell_{\rm FT} ({\phi_k}, {\theta_{k,t}}; z_{k,t}) -  \nabla^2_{{\theta}} \ell_{\rm FT} ({\phi_k}, {\theta_*}; z_{k,t}) \right\Vert^2 + \left\Vert \nabla^2_{{\theta}} \ell_{\rm FT} ({\phi_k}, {\theta_*}; z_{k,t}) - \nabla^2_{{\theta}}\mathcal{L}_{\rm FT}({\phi_k}, {\theta_*}) \right\Vert^2 \right] \\
				\overset{(ii)}{\leq}& 2\frac{B^2\rho^2}{\exp2\beta}\mathbb{E}\left[ \left\Vert \theta_{k,t} - \theta_* \right\Vert^2 \right] + 2\sigma_2^2,
			\end{align*}
			where $(i)$ and $(ii)$ follow from Lemma \ref{lemma:lip}.
			
			Thus, we have
			\begin{align*}
				&\mathbb{E} \left[\left\Vert \frac{\partial \theta_{k,t+1}}{\partial \phi_k} - \frac{\partial \theta_*}{\partial \phi_k} \right\Vert^2\right] + \frac{T^2(2L+M^2)^2}{4\beta^2\mu^2} \mathbb{E}\left[ \left\Vert \theta_{k,t+1} - \theta_* \right\Vert^2\right]\\
				\leq& 3\left(1-\frac{\alpha b\mu}{\exp\beta}\right)^2\mathbb{E}\left[\left\Vert \frac{\partial \theta_{k,t}}{\partial \phi_k} - \frac{\partial \theta_*}{\partial \phi_k} \right\Vert^2\right] + \frac{6\alpha^2B^2T^2}{\beta^2\exp2\beta}\left(L^2 +\frac{M^2\rho^2}{\mu^2}\right) \mathbb{E}\left[ \left\Vert \theta_{k,t} - \theta_* \right\Vert^2\right] \\
				&+ \frac{T^2(2L+M^2)^2}{4\beta^2\mu^2} \mathbb{E}\left[ \left\Vert \theta_{k,t+1} - \theta_* \right\Vert^2\right] + 6\alpha^2 \left( \sigma_1^2 + \frac{M^2 T^2}{\beta^2 \mu^2}\sigma_2^2 \right) \\
				\overset{(i)}{\leq}& 3\left(1-\frac{\alpha b\mu}{\exp\beta}\right)^2 \mathbb{E}\left[\left\Vert \frac{\partial \theta_{k,t}}{\partial \phi_k} - \frac{\partial \theta_*}{\partial \phi_k} \right\Vert^2\right]  + \frac{\alpha^2\sigma^2T^2(2L+M^2)^2}{4\beta^2\mu^2} +  6\alpha^2 \left( \sigma_1^2 + \frac{M^2 T^2}{\beta^2 \mu^2}\sigma_2^2 \right)  \\
				&+ \left( \frac{6\alpha^2B^2T^2}{\beta^2\exp2\beta}\left(L^2 +\frac{M^2\rho^2}{\mu^2}\right) +  \frac{T^2(2L+M^2)^2}{4\beta^2\mu^2} \left( 1 - \frac{2\alpha b\mu BL}{(b\mu + BL)\exp\beta} \right) \right) \mathbb{E}\left[ \left\Vert \theta_{k,t} - \theta_* \right\Vert^2\right] \\
				\overset{(ii)}{\leq}& 3\left(1-\frac{\alpha b\mu}{\exp\beta}\right)^2 \left( \mathbb{E}\left[\left\Vert \frac{\partial \theta_{k,t}}{\partial \phi_k} - \frac{\partial \theta_*}{\partial \phi_k} \right\Vert^2\right] + \frac{T^2(2L+M^2)^2}{4\beta^2\mu^2} \mathbb{E}\left[ \left\Vert \theta_{k,t} - \theta_* \right\Vert^2\right] \right)\\
				&+ \frac{\alpha^2\sigma^2T^2(2L+M^2)^2}{4\beta^2\mu^2} +  6\alpha^2 \left( \sigma_1^2 + \frac{M^2 T^2}{\beta^2 \mu^2}\sigma_2^2 \right),
			\end{align*}
			where $(i)$ follows from Lemma \ref{lemma5}, $(ii)$ follows from a constrain on $\beta$ such that $\beta \leq \log \frac{4\alpha B^2(L^2\mu^2 + M^2\rho^2)}{b\mu (2L+M^2)^2}$.

			Telescoping over $t$ from $0$ to $D-1$, we have
			\begin{align}
				\begin{split}
					&\mathbb{E} \left[\left\Vert \frac{\partial \theta_{k,t+1}}{\partial \phi_k} - \frac{\partial \theta_*}{\partial \phi_k} \right\Vert^2\right] + \frac{T^2(2L+M^2)^2}{4\beta^2\mu^2} \mathbb{E}\left[ \left\Vert \theta_{k,t+1} - \theta_* \right\Vert^2\right] \\
					{\leq}& \left(3\left(1-\frac{\alpha b\mu}{\exp\beta}\right)^{2}\right)^D \left(\left\Vert \frac{\partial \theta_{0,k}}{\partial \phi_k} - \frac{\partial \theta_*}{\partial \phi_k} \right\Vert^2 + \frac{T^2(2L+M^2)^2}{4\beta^2\mu^2}  \left\Vert \theta_{k,0} - \theta_* \right\Vert^2 \right) \\
					& + \left( \frac{\alpha^2\sigma^2T^2(2L+M^2)^2}{4\beta^2\mu^2} +  6\alpha^2 \left( \sigma_1^2 + \frac{M^2 T^2}{\beta^2 \mu^2}\sigma_2^2 \right) \right)\sum_{t=0}^{D-1} \left(3\left(1-\frac{\alpha b\mu}{\exp\beta}\right)^{2}\right)^t
				\end{split}
			\end{align}
			
			Applying the initialization $\theta_{k,0} = \theta_{0}$ at each inner-loop updating interation, we have the result in Eq. (\ref{eq:lemma4}).
			
		\end{proof}

		\begin{lemma} \label{lemma:3}
			Based on Assumption \ref{asp:strong1}, \ref{asp:lip2} and \ref{asp:bound3}, we have for any $k \in \{ 0,1,2,\ldots,K-1\}$,
			\begin{align} \label{eq:lemma3}
				\begin{split}
					&\mathbb{E} \left[\left\Vert \frac{\partial \mathcal{L}_{\rm PL}({\phi_k}, {\theta_{k,D}})}{\partial \phi_k} - \frac{\partial \mathcal{L}_{\rm PL} ({\phi_k}, {\theta_*})}{\partial \phi_{k}} \right\Vert^2\right] \\
					\leq&\frac{2M^2T^2(2L+M^2)^2}{\beta^2\mu^2} \left(3\left(1-\frac{\alpha b\mu}{\exp\beta}\right)^{2}\right)^D \left(\frac{4M^2}{(2L+M^2)^2}  +  \left\Vert \theta_{0} - \theta_* \right\Vert^2\right) \\
					& + 8M^2 \left( \frac{\alpha^2\sigma^2T^2(2L+M^2)^2}{4\beta^2\mu^2} +  6\alpha^2 \left( \sigma_1^2 + \frac{M^2 T^2}{\beta^2 \mu^2}\sigma_2^2 \right) \right)\sum_{t=0}^{D-1} \left(3\left(1-\frac{\alpha b\mu}{\exp\beta}\right)^{2}\right)^t
				\end{split}
			\end{align}
		\end{lemma}
		
		\begin{proof}
			Based on $\frac{\partial \mathcal{L}_{\rm PL} (\phi_k, \theta_*)}{\partial \phi_{k}} = \frac{\partial \mathcal{L}_{\rm PL}(\theta_*)}{\partial \phi_{k}} = \frac{\partial \theta_*}{\partial \phi_k} \nabla_{{\theta}}  \mathcal{L}_{\rm PL}(\phi_k, \theta_*)$, and similarly $\frac{\partial \mathcal{L}_{\rm PL}(\phi_k, \theta_{k,D})}{\partial \phi_k} = \frac{\partial \theta_{k,D}}{\partial \phi_k} \nabla_{{\theta}}  \mathcal{L}_{\rm PL}({\phi_k}, {\theta_{k,D}})$, using the triangle inequality and Jensen's inequality, we have
			\begin{align*}
				&\left\Vert \frac{\partial \mathcal{L}_{\rm PL}({\phi_k}, {\theta_{k,D}})}{\partial \phi_k} - \frac{\partial \mathcal{L}_{\rm PL} ({\phi_k}, {\theta_*})}{\partial \phi_{k}} \right\Vert = \left\Vert \frac{\partial {\theta_{k,D}}}{\partial \phi_k} \nabla_{{\theta}}  \mathcal{L}_{\rm PL}({\phi_k}, {\theta_{k,D}}) - \frac{\partial {\theta_*}}{\partial \phi_k} \nabla_{{\theta}}  \mathcal{L}_{\rm PL}({\phi_k}, {\theta_*}) \right\Vert \\
				\leq& \left\Vert  \frac{\partial {\theta_{k,D}}}{\partial \phi_k} - \frac{\partial {\theta_*}}{\partial \phi_k} \right\Vert \left\Vert \nabla_{{\theta}}  \mathcal{L}_{\rm PL}({\phi_k}, {\theta_{k,D}}) \right\Vert + \left\Vert \frac{\partial {\theta_*}}{\partial \phi_k}  \right\Vert \left\Vert \nabla_{{\theta}}  \mathcal{L}_{\rm PL}({\phi_k}, {\theta_{k,D}}) - \nabla_{{\theta}}  \mathcal{L}_{\rm PL}({\phi_k}, {\theta_*}) \right\Vert \\
				\overset{(i)}{\leq}& 2M \left\Vert \frac{\partial {\theta_{k,D}}}{\partial \phi_k} - \frac{\partial {\theta_*}}{\partial \phi_k} \right\Vert + \left\Vert \frac{\partial {\theta_*}}{\partial \phi_k}  \right\Vert \left\Vert \nabla_{{\theta}}  \mathcal{L}_{\rm PL}({\phi_k}, {\theta_{k,D}}) - \nabla_{{\theta}}  \mathcal{L}_{\rm PL}({\phi_k}, {\theta_*}) \right\Vert \\
				\overset{(ii)}{\leq}& 2M \left\Vert \frac{\partial {\theta_{k,D}}}{\partial \phi_k} - \frac{\partial {\theta_*}}{\partial \phi_k} \right\Vert + \frac{MT(2L+M^2)}{\beta \mu} \Vert \theta_{k,D} - {\theta_*} \Vert,
			\end{align*}
			where $(i)$ follows from Jensen's inequality and Lemma \ref{lemma3}, $(ii)$ follows from Jensen's inequality and Lemma \ref{lemma3}, \ref{lemma4}.
			
			Using the fact $(a+b)^2 \leq 2(a^2 + b^2)$ and taking expectation on both sides, we have
			\begin{align*}
				&\mathbb{E}\left[ \left\Vert \frac{\partial \mathcal{L}_{\rm PL}({\phi_k}, {\theta_{k,D}})}{\partial \phi_k} - \frac{\partial \mathcal{L}_{\rm PL} ({\phi_k}, {\theta_*})}{\partial \phi_{k}} \right\Vert^2 \right] \\
				\leq& 8M^2 \left(  \mathbb{E}\left[\left\Vert \frac{\partial {\theta_{k,D}}}{\partial \phi_k} - \frac{\partial {\theta_*}}{\partial \phi_k} \right\Vert^2\right] + \frac{T^2(2L+M^2)^2}{4\beta^2\mu^2} \mathbb{E}\left[\Vert \theta_{k,D} - {\theta_*} \Vert^2\right] \right),
			\end{align*}
			and using Lemma \ref{lemma:4}, we have Eq. (\ref{eq:lemma3}).
		\end{proof}
		
		\subsection{Proof of Theorem \ref{thm:1}}
		
		\begin{proof}
			Let $\mathbf{L}(\phi_{k}) = \mathbb{E}_{T\sim \mathcal{D}_{\mathcal{T}}}[\mathcal{L}_{\rm PL}(\phi_k, \theta^{T}_*)]$ and let $\widetilde{\nabla}\mathbf{L}(\phi_{k}) = \frac{\partial \mathbb{E}_{T \sim \mathcal{D}_{\mathcal{T}}}[\mathcal{L}_{\rm PL}(\phi_k, {\theta}_{k,D}^T)]}{\partial \phi_{k}}$, $\widehat{\nabla}\mathbf{L}(\phi_{k}) = \frac{\partial \ell_{\rm PL}(\phi_k, {\theta}_{k,D}^T; z_k)}{\partial \phi_{k}}$. Adopting Lemma \ref{lemma:smooth-phi} with Jensen's inequality, we have the $P_{\phi}$-smoothness of $\mathbf{L}(\phi_{k})$, and thus
			\begin{align*}
				&{\bf L}(\phi_{k+1}) \\
				\leq& {\bf L}(\phi_{k}) + \left\langle \nabla {\bf L}(\phi_{k}), \phi_{k+1}-\phi_{k} \right\rangle + \frac{P_{\phi}}{2}\Vert \phi_{k+1}-\phi_{k} \Vert^2 \\
				=& {\bf L}(\phi_{k}) - \eta\left\langle \nabla {\bf L}(\phi_{k}), 
				\widehat{\nabla}{\bf L}(\phi_{k}) \right\rangle + \frac{\eta^2P}{2} \left\Vert \widehat{\nabla}{\bf L}(\phi_{k}) \right\Vert^2 \\
				\leq& {\bf L}(\phi_{k}) - \eta \left\langle \nabla {\bf L}(\phi_{k}), \widehat{\nabla} {\bf L}(\phi_{k}) - \nabla {\bf L}(\phi_{k}) \right\rangle - \eta\left\Vert {\nabla}{\bf L}(\phi_{k}) \right\Vert^2 \\
				& + \eta^2 P\Vert \nabla {\bf L}(\phi_{k}) \Vert^2 + \eta^2P_{\phi}\Vert  \nabla {\bf L}(\phi_{k}) -  \widehat{\nabla} {\bf L}(\phi_{k}) \Vert^2 \\
				\leq& {\bf L}(\phi_{k}) - \left( \frac{\eta}{2} - \eta^2P_{\phi} \right) \Vert \nabla {\bf L}(\phi_{k})\Vert^2 + \left( \frac{\eta}{2} + \eta^2P_{\phi} \right)  \Vert {\nabla}{\bf L}(\phi_{k}) - \widehat{\nabla}{\bf L}(\phi_{k}) \Vert^2 \\
				\overset{(i)}{\leq}& {\bf L}(\phi_{k}) - \left( \frac{\eta}{2} - \eta^2P_{\phi} \right) \Vert \nabla {\bf L}(\phi_{k})\Vert^2 + \left( \frac{\eta}{2} + \eta^2P_{\phi} \right) \mathbb{E}_{T \sim \mathcal{D}_{\mathcal{T}}} \left[ \left\Vert \frac{\partial \mathcal{L}_{\rm PL}(\phi_k, {\theta}^{T}_*)}{\partial \phi_{k}} - \widehat{\nabla}\mathbf{L}(\phi_{k}) \right\Vert^2\right],
			\end{align*}
			where $(i)$ follows from Jensen's inequality.
			Taking expectation on both sides and rearranging the terms conditioning on $\eta(\frac{1}{2}-\eta P_{\phi}) \geq 0$, we have
			\begin{align*}
				&\left( \frac{\eta}{2} - \eta^2P_{\phi} \right) \mathbb{E} \left[ \Vert \nabla {\bf L}(\phi_{k})\Vert^2 \right] \\
				\leq& \mathbb{E}\left[ {\bf L}(\phi_{k}) \right] - \mathbb{E}\left[ {\bf L}(\phi_{k+1}) \right] + \left( \frac{\eta}{2} + \eta^2P_{\phi} \right) \mathbb{E}_{T\sim \mathcal{D}_{\mathcal{T}}} \mathbb{E} \left[ \left\Vert \frac{\partial \mathcal{L}_{\rm PL}(\phi_k, {\theta}^{T}_*)}{\partial \phi_{k}} - \widehat{\nabla}\mathbf{L}(\phi_{k}) \right\Vert^2\right] \\
				\leq& \mathbb{E}\left[ {\bf L}(\phi_{k}) \right] - \mathbb{E}\left[ {\bf L}(\phi_{k+1}) \right] + \left( \frac{\eta}{2} + \eta^2P_{\phi} \right) \mathbb{E}_{T\sim \mathcal{D}_{\mathcal{T}}}\mathbb{E} \left[ \left\Vert \frac{\partial \mathcal{L}_{\rm PL}(\phi_k, \theta_{k,D}^T)}{\partial \phi_k} - \widehat{\nabla}\mathbf{L}(\phi_{k}) \right\Vert^2\right] \\
				&+ \left( \frac{\eta}{2} + \eta^2P_{\phi} \right) \mathbb{E}_{T\sim \mathcal{D}_{\mathcal{T}}}\mathbb{E} \left[\left\Vert \frac{\partial \mathcal{L}_{\rm PL}(\phi_k, \theta_{k,D}^T)}{\partial \phi_k} - \frac{\partial \mathcal{L}_{\rm PL}(\phi_k, {\theta}^{T}_*)}{\partial \phi_{k}} \right\Vert^2\right] \\
				\overset{(ii)}{\leq}& \mathbb{E}\left[ {\bf L}(\phi_{k}) \right] - \mathbb{E}\left[ {\bf L}(\phi_{k+1}) \right] + \left( \frac{\eta}{2} + \eta^2P_{\phi} \right) \sigma_3 \\
				& +\!\! 8M^2\!\!\left( \frac{\eta}{2} + \eta^2P_{\phi} \right)\!\!\left( \frac{\alpha^2\sigma^2T^2(2L+M^2)^2}{4\beta^2\mu^2}\!\! +\!\!  6\alpha^2\!\! \left( \sigma_1^2 + \frac{M^2 T^2}{\beta^2 \mu^2}\sigma_2^2 \right)\!\! \right)\!\!\sum_{t=0}^{D-1}\!\! \left(3\left(1-\frac{\alpha b\mu}{\exp\beta}\right)^{2}\right)^t\\
				&+ \frac{2M^2T^2(2L+M^2)^2}{\beta^2\mu^2}\left( \frac{\eta}{2} + \eta^2P_{\phi} \right)   \left(3\left(1-\frac{\alpha b\mu}{\exp\beta}\right)^{2}\right)^D \left(\frac{4M^2}{(2L+M^2)^2} +  \left\Vert \theta_{0} - \theta_* \right\Vert^2\right),
			\end{align*}
			where $(ii)$ follows from Lemma \ref{lemma:3}.
			
			Telescoping the above inequality over $k$ from $0$ to $K-1$ yields
			\begin{align*}
				&\left( \frac{\eta}{2} - \eta^2P_{\phi} \right)\frac{1}{K}\sum_{k=0}^{K-1} \mathbb{E} \left[ \Vert \nabla {\bf L}(\phi_{k})\Vert^2 \right] \\
				\leq& \frac{{\bf L}(\phi_0) - \inf_{\phi}{\bf L}(\phi) }{K} + \left( \frac{\eta}{2} + \eta^2P_{\phi} \right) \sigma_3^2  \\
				& +\!\! 8M^2\!\!\left( \frac{\eta}{2} \!\!+\!\! \eta^2P_{\phi} \right)\!\!\left(  \frac{\alpha^2\sigma^2T^2(2L+M^2)^2}{4\beta^2\mu^2} \!\!+\!\!  6\alpha^2 \!\!\left( \sigma_1^2 + \frac{M^2 T^2}{\beta^2 \mu^2}\sigma_2^2 \right)\!\! \right)\!\!\sum_{t=0}^{D-1}\!\! \left(3\left(1-\frac{\alpha b\mu}{\exp\beta}\right)^{2}\right)^t \\
				&+ \frac{2M^2T^2(2L+M^2)^2}{\beta^2\mu^2}\left( \frac{\eta}{2} + \eta^2P_{\phi} \right) \left(3\left(1-\frac{\alpha b\mu}{\exp\beta}\right)^{2}\right)^D \left(\frac{4M^2}{(2L+M^2)^2} +  \left\Vert \theta_{0} - \theta_* \right\Vert^2\right) \\
				\overset{(i)}{\leq}&  \frac{{\bf L}(\phi_0) - \inf_{\phi}{\bf L}(\phi) }{K} + \left( \frac{\eta}{2} + \eta^2P_{\phi} \right) \sigma_3^2 \\ 
				& +\!\!\frac{8M^2}{\kappa}\!\!\left( \frac{\eta}{2} + \eta^2P_{\phi} \right)\!\! \left(\!\!  \frac{\alpha^2\sigma^2T^2(2L+M^2)^2}{4\beta^2\mu^2}\!\! +\!\!  6\alpha^2\!\! \left( \sigma_1^2 + \frac{M^2 T^2}{\beta^2 \mu^2}\sigma_2^2 \right)\!\! \right) \\
				&+ \frac{2M^2T^2(2L+M^2)^2}{\beta^2\mu^2} \left( \frac{\eta}{2} + \eta^2P_{\phi} \right),
			\end{align*}
			where $(i)$ follows from that $D \geq - \frac{\log \left(\frac{4M^2}{(2L+M^2)^2}  +  \left\Vert \theta_{0} - \theta_* \right\Vert^2\right)}{\log \left(3\left(1-\frac{\alpha b\mu}{\exp\beta}\right)^{2}\right)}$ implies $\left(3\left(1-\frac{\alpha b\mu}{\exp\beta}\right)^{2}\right)^D\!\! \left(\frac{4M^2}{(2L+M^2)^2}  +  \left\Vert \theta_{0} - \theta_* \right\Vert^2\right) \leq 1$, $\beta \leq \log\left( \alpha b\mu /  1 - \sqrt{\frac{1-\kappa}{3}}\right)$ implies $\sum_{t=0}^{D-1} \left(3\left(1-\frac{\alpha b\mu}{\exp\beta}\right)^{2}\right)^t \leq \frac{1}{1-3\left(1-\frac{\alpha b\mu}{\exp\beta}\right)^2} \leq \frac{1}{\kappa}$, for some positive constant $\kappa < 1$.
			
			Then, setting $\eta = \frac{1}{4P_{\phi}}$, we have
			\begin{align*}
				&\frac{1}{K}\sum_{k=0}^{K-1} \mathbb{E} \left[ \Vert \nabla {\bf L}(\phi_{k})\Vert^2 \right] \\ \leq& \frac{16P_{\phi}\left({\bf L}(\phi_0) - \inf_{\phi}{\bf L}(\phi)\right) }{K} + 3\sigma_3^2 + \frac{6M^2T^2(2L+M^2)^2}{\beta^2\mu^2} \\
				& + \frac{24M^2}{\kappa}\left(  \frac{\alpha^2\sigma^2T^2(2L+M^2)^2}{4\beta^2\mu^2} +  6\alpha^2 \left( \sigma_1^2 + \frac{M^2 T^2}{\beta^2 \mu^2}\sigma_2^2 \right) \right) \\
				=& \frac{16P_{\phi}\left({\bf L}(\phi_0) - \inf_{\phi}{\bf L}(\phi)\right) }{K} + \left( (2L+M^2)^2 + \frac{4\alpha^2\sigma^2(2L+M^2)^2}{\kappa } + \frac{24M^2\sigma_2^2}{\kappa} \right)\frac{6M^2T^2}{\beta^2 \mu^2}\\
				& + 3\sigma_3^3 + \frac{144 M^2 \alpha^2 \sigma_1^2}{\kappa}
			\end{align*}
			
			Then, denoting $ \tau = 3\sigma_3^3 + \frac{144 M^2 \alpha^2 \sigma_1^2}{\kappa} $, we obtain the result of Theorem \ref{thm:1}.
		\end{proof} 
	\end{footnotesize}
		
		\section{Further Details on the Experimental Setup}
		
		\subsection{Statistics of datasets} \label{apdx:datastatic}
		Tab. \ref{tbl:datastatic} lists statistics of the used training and evaluation data for CSG and KAG across different domains.
		
		\begin{table*}[h] 
			\centering
			\caption{Dataset statistics.}
			\scalebox{.9}{
				\begin{tabular}{lcccc}
					\toprule
					\multirow{2}{*}{\bf Domain} & \multicolumn{2}{c}{ Training Data} &
					\multicolumn{2}{c}{Evaluation Data} \\
					\cmidrule(lr){2-3}	\cmidrule(lr){4-5}
					& {Fine-Tuning} & {Pref. Learn.} & {Reference} &{Preference} \\
					\midrule
					\multicolumn{5}{c}{\bf CSG} \\
					\midrule
					{\tt book} &1,000& 500& 500& 500 \\
					{\tt dvd} &1,000& 500 & 500& 500  \\
					{\tt electronics} &1,000& 500&500 &500 \\
					{\tt kitchen} &1,000& 500& 500 &500  \\
					\midrule
					\multicolumn{5}{c}{\bf KAG} \\
					\midrule
					{\tt askacademia} ({\tt ac}) &31,450 &2,095 &1,708 &1,708 \\
					{\tt askanthropology} ({\tt an})&3,910 &203 &268 &268 \\
					{\tt askbaking} ({\tt ba})&44,007 &2,096 &268 &268 \\
					{\tt askcarguys} ({\tt ca})&3,227 &159 &117 &117 \\
					{\tt askculinary} ({\tt cu})&45,710 &2,094 &2,563 &2,563 \\
					{\tt askdocs} ({\tt do})&6,449 &315 &455 &455 \\
					{\tt askengineers} ({\tt en})&57,096 &3,154 &2,638	&2,638 \\
					{\tt askhistorians} ({\tt hi})&3,264 &113 &164 &164 \\
					{\tt askhr} ({\tt hr})&8,295 &641 &395 &395 \\
					{\tt askphilosophy} ({\tt phi})&10,307 &608 &677 &677 \\
					{\tt askphysics} ({\tt phy})&7,364 &409 &587 &587 \\
					{\tt askscience} ({\tt sci})&13,316 &899 &977 &977 \\
					{\tt asksciencefiction} ({\tt scif})&29,382 &1,576 &1,987 &1,987 \\
					{\tt asksocialscience} ({\tt so})&2,706 &147 &188 &188 \\
					{\tt askvet} ({\tt ve})&3,300 &170 &224 &224 \\
					{\tt changemyview} ({\tt ch})&38,173 &1,637 &1,836 &1,836 \\
					{\tt explainlikeimfive} ({\tt ex})&19,592 &1,014 &1,070 &1,070 \\
					{\tt legaladvice} ({\tt le})&21,170 &1,106 &1,011 &1,011 \\	
					\bottomrule
			\end{tabular}}
			\label{tbl:datastatic}
		\end{table*}
		
		\subsection{GPT-4 prompts for generating preference data for sentiment generation} \label{apdx:gpt4gensent}
		We use GPT-4 to generate the preference data for CSG with the original Amazon Review dataset. The generating prompts are provided as follows, in which {\tt \{book/dvd/electronics/kitchen\}} and {\tt \{positive/negative\}} denote the alternative terms based on the domain and sentiment label of the input review, respectively.
		
		{\tt Extract a prefix for the \{book/dvd/electronics/kitchen\} review (<reference>) or add a prefix for this review if needed. Then, based on the prefix write another review that is not as {positive/negative} as the given review.} \\ 
		\\
		{\tt <reference>: [the provided review]} \\ 
		\\
		{\tt Line by line provides the generated result using the following format:}\\
		\\
		{\tt <prefix>: [provides the extracted/added prefix]} \\
		\\
		{\tt <origin>: [provides the new reference with the provided prefix]} \\
		\\
		{\tt <augment>: [provides the new review with the same provided prefix]} \\
		\\
		{\tt <label>: [if <origin> is more positive than <augment> state "0", otherwise "1"]} \\
		
		\subsection{GPT-4 prompts for computing KAG win rates}
		An important aspect of our experimental configuration involves assessing GPT-4 win rates. This section provides the prompts utilized to generate win rates for the KAG experiments. 
		
		{\tt Which of the following answers (<Answer 1> or <Answer 2>) is more correct and useful to solve the question (<Question>)?}\\
		\\
		{\tt <Question>: [the prompt]}\\
		\\
		{\tt <Answer 1>: [reference]}\\
		\\
		{\tt <Answer 2>: [model generations]} \\
		\\
		{\tt FIRST provide a one-sentence comparison of the two answers on which is more correct and useful to solve the question. SECOND explain in detail why the selected answer is more correct and useful compared with the other. THIRD on a new line, state only "1" or "2" to indicate which answer is more correct and useful.}

		\section{Additional Experimental Results}
		\subsection{Additional Evaluation with Preference Data} \label{apdx:}
		As the experiments in the main paper verify the CSG, we illustrate the results on KAG in this section. 
		\noindent\textbf{Training and evaluation PL accuracies for KAG.} As illustrated in Fig. \ref{fig:knowgen2}, the trends of training PL loss and PL accuracy against training steps for all six groups of held-out domains are similar, with only small fluctuations. The evaluation PL accuracy against the training steps for all four domains is also similar, and our approach achieves a significantly higher plateau compared with other baselines for each domain.  
		
		\begin{figure}[h] 
			\centering
			\subfigure[Training PL loss and accuracy for held-out domains: {\tt ac,an,ba}.]{
				\begin{minipage}[h]{0.3\textwidth}
					\centering
					\includegraphics[width=4cm]{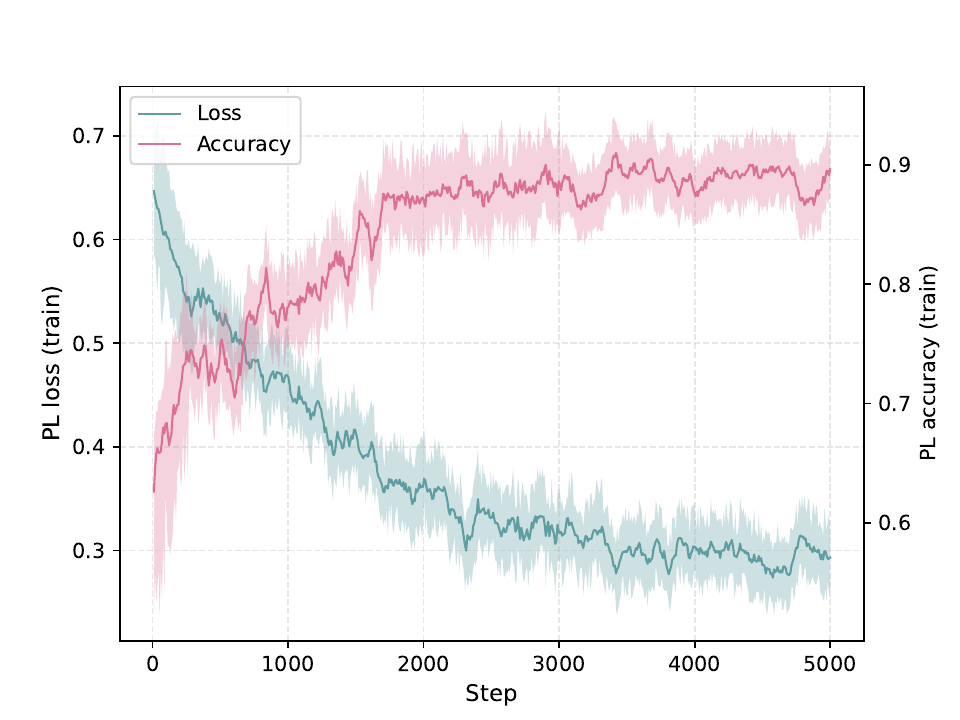}
			\end{minipage}}
			\subfigure[Training PL loss and accuracy for held-out domain: {\tt ca,cu,do}.]{
				\begin{minipage}[h]{0.3\textwidth}
					\centering
					\includegraphics[width=4cm]{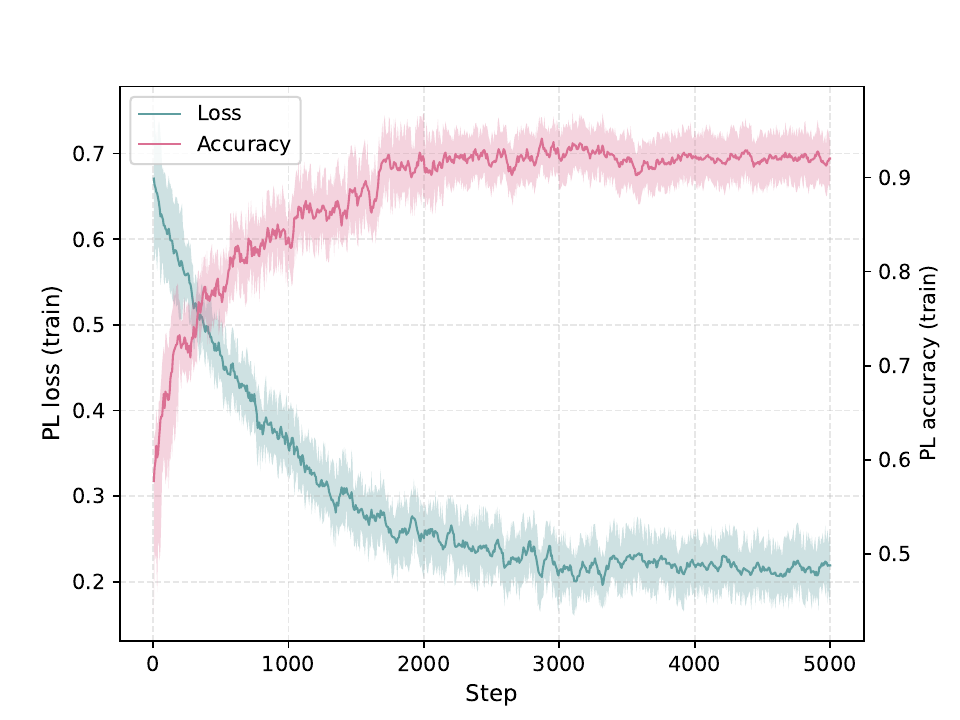}
			\end{minipage}}
			\subfigure[Training PL loss and accuracy for held-out domain: {\tt ne,hi,hr}.]{
				\begin{minipage}[h]{0.3\textwidth}
					\centering
					\includegraphics[width=4cm]{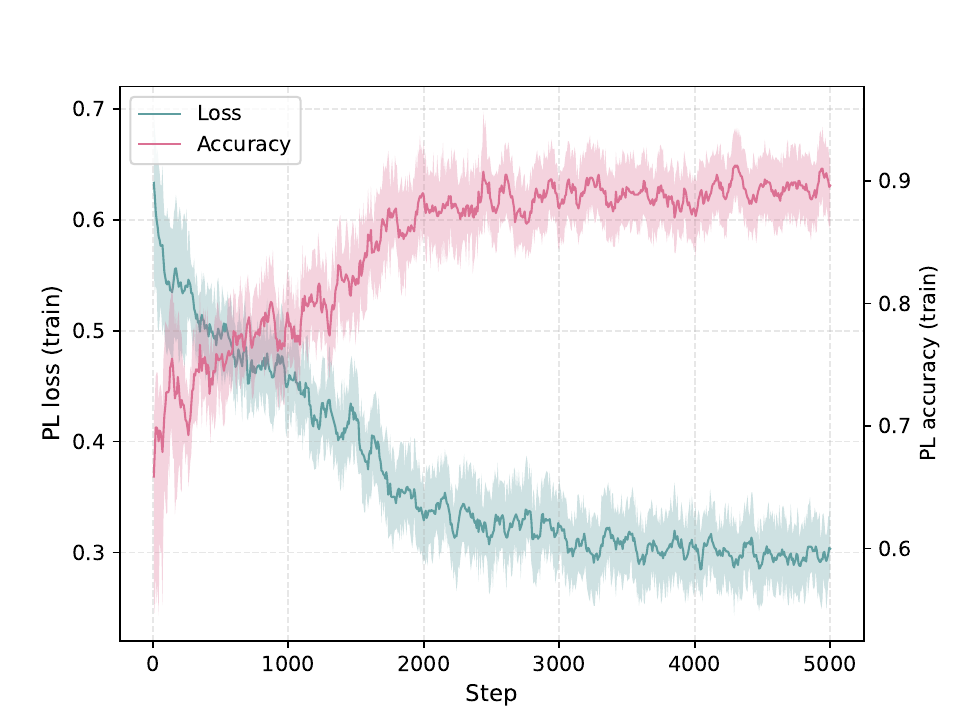}
			\end{minipage}}
			\subfigure[Training PL loss and accuracy for held-out domain: {\tt phi,phy,sci}.]{
				\begin{minipage}[h]{0.3\textwidth}
					\centering
					\includegraphics[width=4cm]{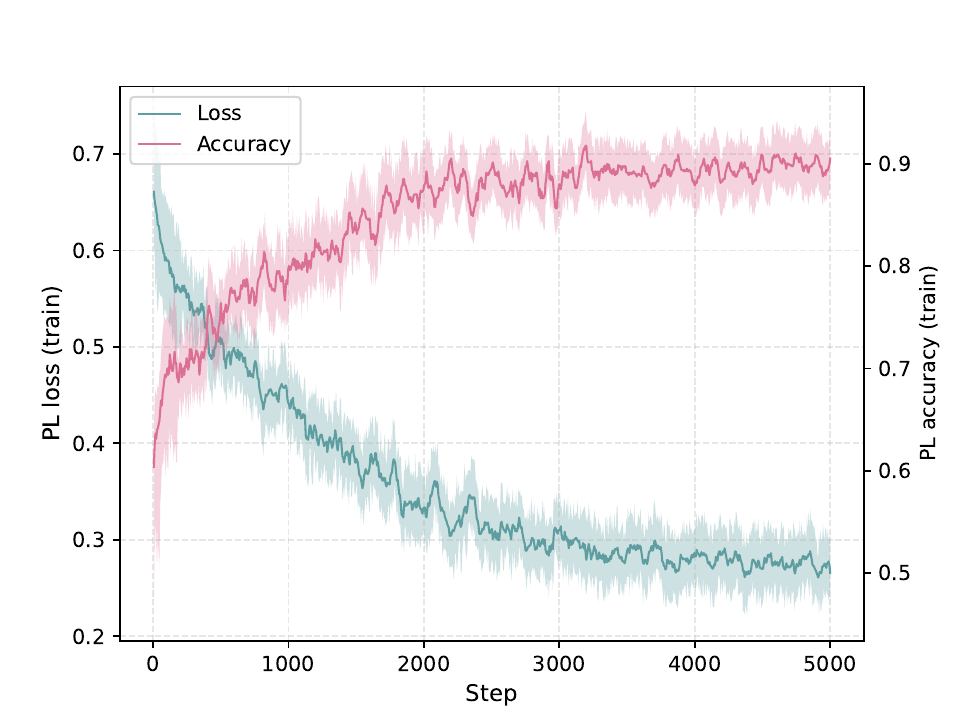}
			\end{minipage}}
			\subfigure[Training PL loss and accuracy for held-out domain: {\tt scif,so,ve}.]{
				\begin{minipage}[h]{0.3\textwidth}
					\centering
					\includegraphics[width=4cm]{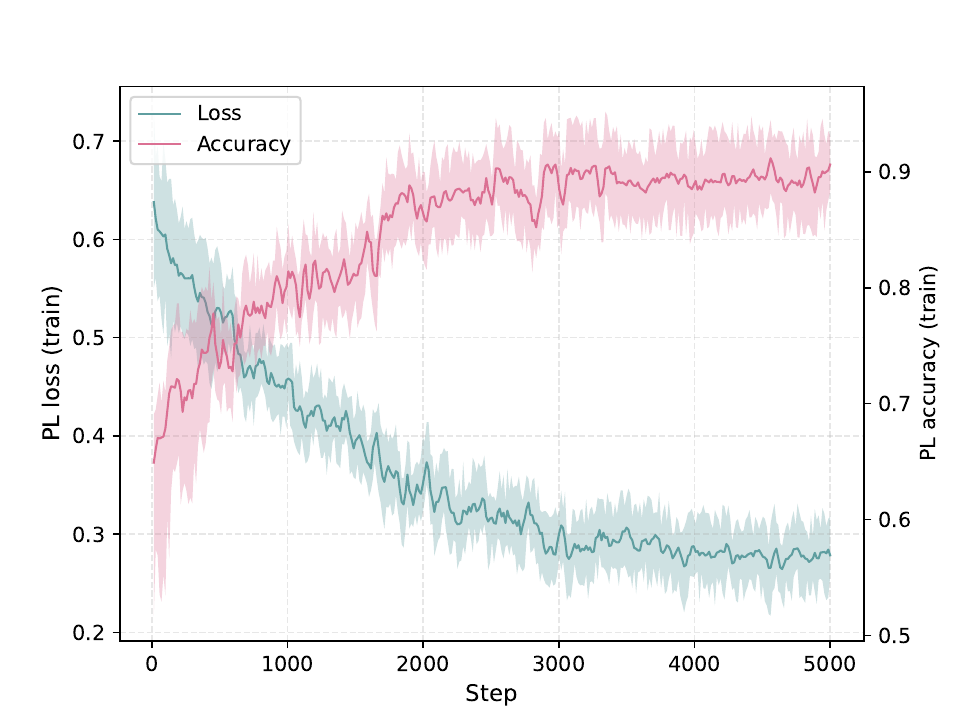}
			\end{minipage}}
			\subfigure[Training PL loss and accuracy for held-out domain: {\tt ch,ex,le}.]{
				\begin{minipage}[h]{0.3\textwidth}
					\centering
					\includegraphics[width=4cm]{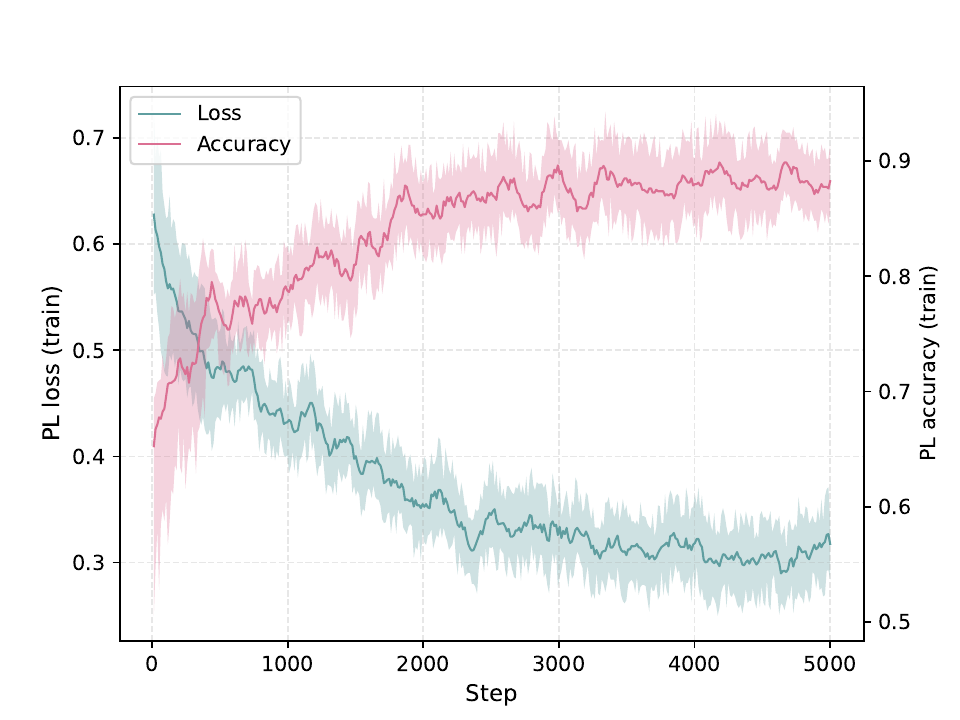}
			\end{minipage}}
			\subfigure[Evaluation PL accuracy on held-out domains: {\tt ac\&an\&ba}.]{
				\begin{minipage}[h]{0.3\textwidth}
					\centering
					\includegraphics[width=4cm]{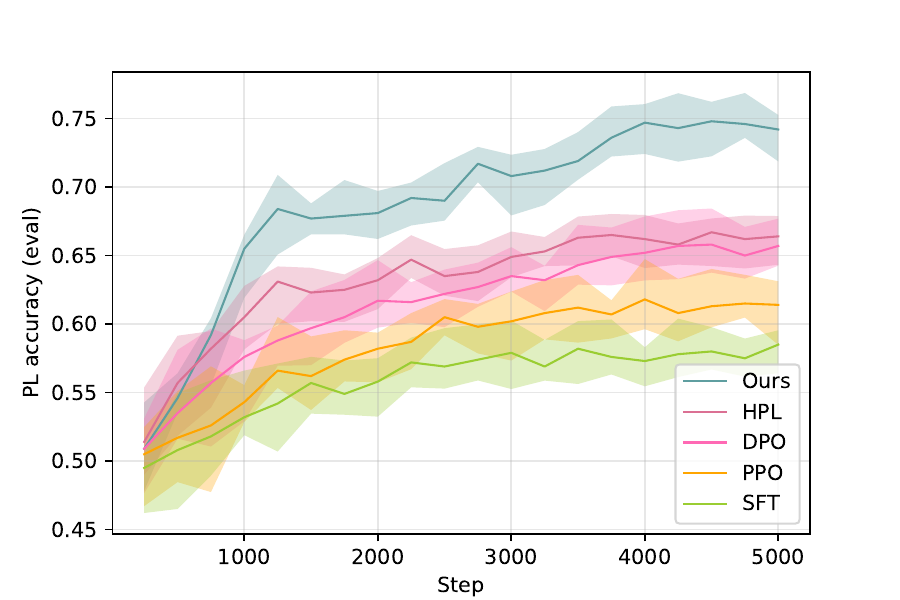}
			\end{minipage}} 
			\subfigure[Evaluation PL accuracy on held-out domain: {\tt ca,cu,do}.]{
				\begin{minipage}[h]{0.3\textwidth}
					\centering
					\includegraphics[width=4cm]{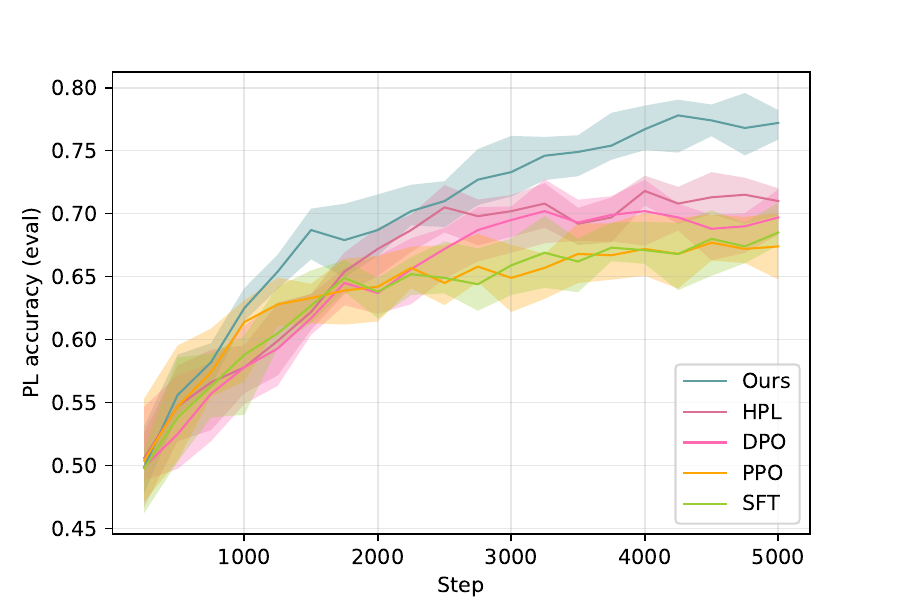}
			\end{minipage}}
			\subfigure[Evaluation PL accuracy on held-out domain: {\tt ne,hi,hr}.]{
				\begin{minipage}[h]{0.3\textwidth}
					\centering
					\includegraphics[width=4cm]{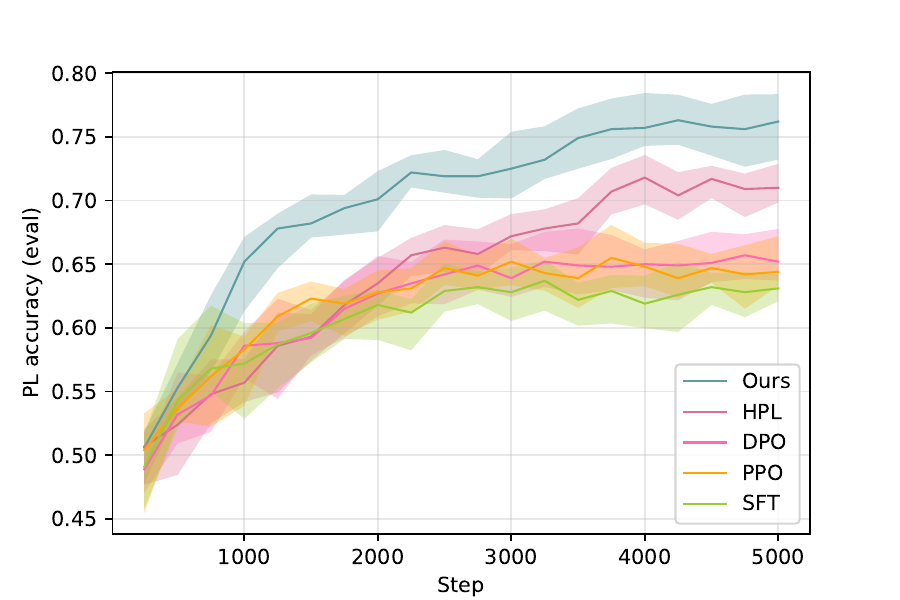}
			\end{minipage}}  
			\subfigure[Evaluation PL accuracy on held-out domain: {\tt phi,phy,sci}.]{
				\begin{minipage}[h]{0.3\textwidth}
					\centering
					\includegraphics[width=4cm]{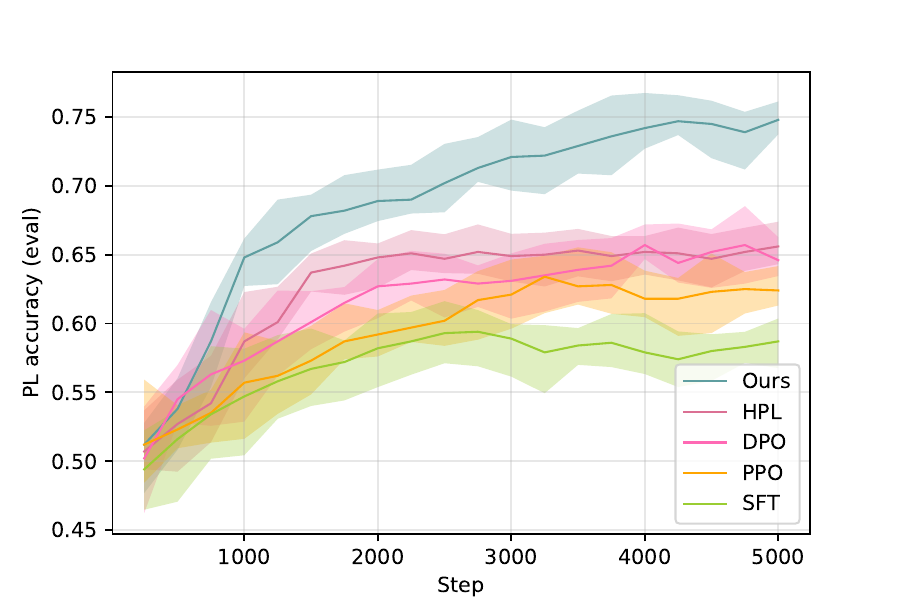}
			\end{minipage}}
			\subfigure[Evaluation PL accuracy on held-out domain: {\tt scif,so,ve}.]{
				\begin{minipage}[h]{0.3\textwidth}
					\centering
					\includegraphics[width=4cm]{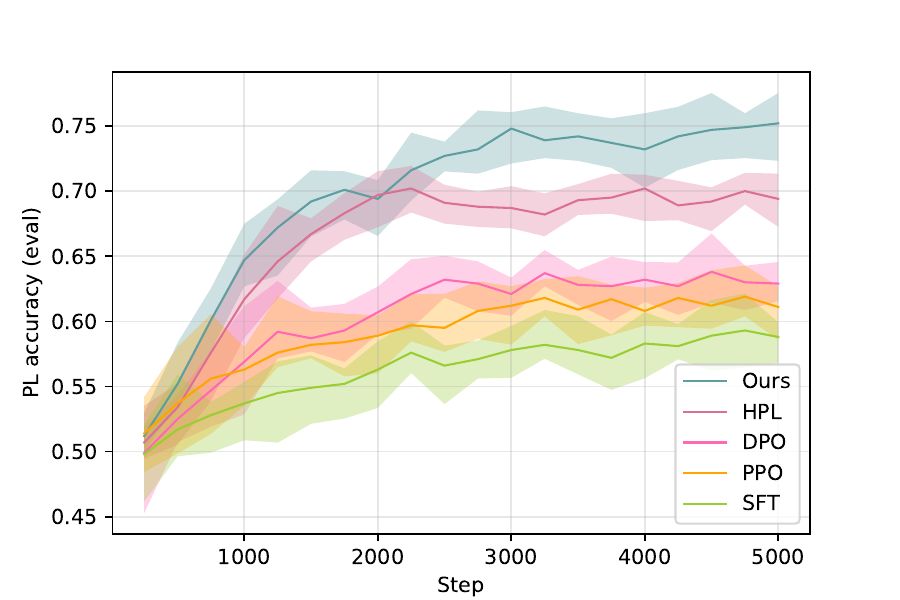}
			\end{minipage}}
			\subfigure[Evaluation PL accuracy on held-out domain: {\tt ch,ex,le}.]{
				\begin{minipage}[h]{0.3\textwidth}
					\centering
					\includegraphics[width=4cm]{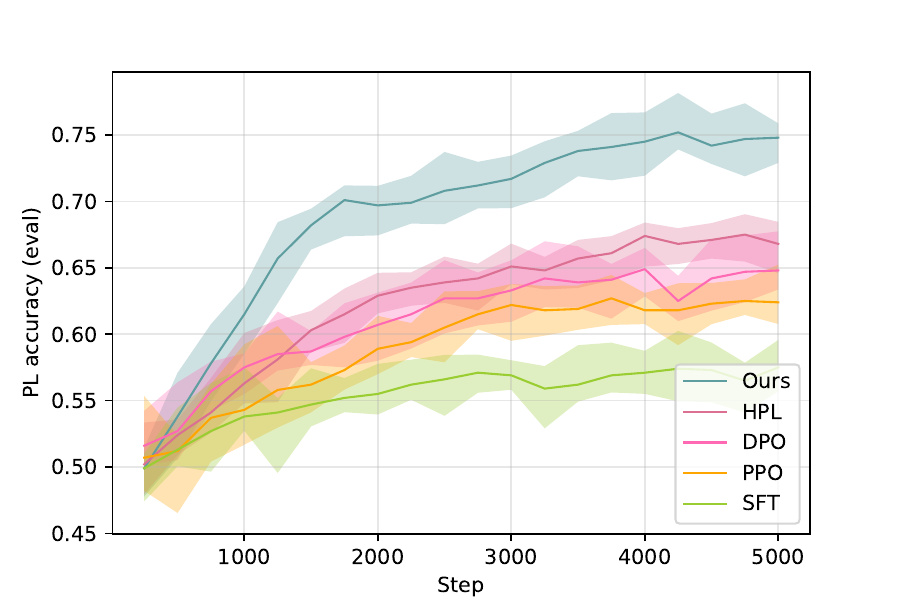}
			\end{minipage}}
			\caption{PL accuracy $\mathcal{A}_{\rm PL}$ against the training steps for KAG.} \label{fig:knowgen2}
		\end{figure}
		
		\subsection{Generation examples} \label{apdx:genexample}
		\noindent\textbf{Examples on sentiment generalization.}
		We present examples of comparison between our method and the SFT, PPO and DPO baselines on CSG. Tab. \ref{tbl:sentimentexample} shows generation examples by different methods for a book review.
		
		\begin{table}[h] 
			\centering
			\caption{Sample sentiment generation to a {\tt book} review from the Amazon Review test set. The examples of all the methods are generated with a temperature of $0.25$.}
			\scalebox{1.}{
				\begin{tabular}{lp{10cm}}
					\toprule
					\multirow{3}{*}{\bf Prompt} & {\bf Instruction.} Continue to write the book review with the <prefix>, and make it have a positive sentiment. \\
					& {\bf Input.} <prefix>: Sphere by Michael Crichton is an excellent novel; this was \\
					\midrule
					{\bf SFT} & Sphere by Michael Crichton is an excellent novel; this was undoubtedly captivating, showcasing Crichton's storytelling finesse. Norman Johnson's Pacific Ocean odyssey brims with mystery and allure, offering a compelling read. While not as deeply researched as some of Crichton's works, "Sphere" still delivers intriguing insights. Overall, a recommended pick for fans of science fiction and adventure.\\
					\midrule
					{\bf PPO} & Sphere by Michael Crichton is an excellent novel; this was truly an immersive and captivating read. Crichton's storytelling prowess shines through in this thrilling narrative, keeping readers engaged from start to finish. Norman Johnson's journey is filled with intrigue and suspense, making for an unforgettable adventure beneath the ocean's surface. Overall, it's a highly recommended read for fans of science fiction and psychological thrillers alike.
					\\
					\midrule
					{\bf DPO} & Sphere by Michael Crichton is an excellent novel; this was an engaging and suspenseful narrative that keeps readers hooked. Norman Johnson's journey into the depths of the Pacific Ocean is filled with intrigue and mystery, making it a captivating read. Overall, it's a compelling read that will appeal to fans of science fiction and adventure.
					\\
					\midrule
					{\bf Ours} & Sphere by Michael Crichton is an excellent novel; this was undeniably a literary marvel that surpasses all expectations. Crichton's masterful storytelling transcends the boundaries of imagination, weaving a narrative so compelling that it becomes impossible to put down. Norman Johnson's journey unfolds with gripping intensity, taking readers on an exhilarating adventure into the unknown depths of the Pacific Ocean. The revelation of the massive spacecraft hidden beneath the waves sets the stage for an epic exploration of science, psychology, and human nature. In every page, Crichton's brilliance shines brightly, illuminating the path for readers to embark on an unforgettable literary journey. In conclusion, "Sphere" is more than just a novel; it is a masterpiece—an absolute must-read for anyone seeking to be enthralled by the magic of storytelling.
					\\
					\bottomrule
			\end{tabular}}
			\label{tbl:sentimentexample}
		\end{table}
		
		\noindent\textbf{Examples on answer generation with GPT-4 judgement.}
		We show examples of KAG in Tab. \ref{tbl:exampleanswer}, we compare our method with a strong baseline PPO on answer generation for a question from the {\tt science} domain, and show the GPT-4 judgement on which answer of ours and the preference of test set is more correct and useful. 
		\begin{table}[h] 
			\centering 
			\caption{Sample answer generation for a question of {\tt science} from the SHP test set. All the methods use greedy strategy for example generation. We use GPT-4 to judge whether our method's generation is more correct and useful to solve the question than the preferred answer in the test set.}
			\scalebox{1.}{
				\begin{tabular}{lp{8cm}}
					\toprule
					{\bf Prompt} & If you were a human floating towards the sun, at what distance from the sun would you feel an Earth-like temperature?
					\\
					\midrule
					{\bf Preference} & Over at NASA, they mention that the  > temperature of the orbiting Space Station's Sun-facing side would soar to 250 degrees F (121 C),  so I guess if you have a clear view of the sun, you'd need to be further away than the Earth is.  I also found this, which states  > At the Earth's distance from the sun, a space thermometer with roughly half its surface is absorbing sunlightwould register 45 degrees Fahrenheit.
					\\
					\midrule
					{\bf PPO} & A naked human wouldn't feel balanced anywhere because our bodies don't spread heat well. One side would burn, and the other would freeze. You'd die no matter how far you are from the sun, even without the sun's radiation and vacuum effects. \\
					\midrule
					{\bf Ours} & In reality, it's impossible for humans to experience Earth-like temperatures at any distance from the sun. The intense heat emitted by the sun would quickly overwhelm the human body's capacity to regulate temperature, leading to severe burns or even evaporation. Thus, while the concept is intriguing, the harsh realities of the solar environment make it unlikely for humans to encounter Earth-like temperatures under such conditions.\\
					\midrule
					{\bf GPT-4 Judgement} &
					<Answer 2> {\bf [Ours]} is more correct and useful to solve the question. While <Answer 1> {\bf [Preference]} provides some information about temperatures in space, it does not directly address the question of at what distance from the sun a human would feel an Earth-like temperature. Additionally, <Answer 1> {\bf [Preference]} does not consider the complexities of human physiology and the extreme conditions present near the sun. <Answer 2> {\bf [Ours]}, on the other hand, acknowledges the impracticality of experiencing Earth-like temperatures near the sun due to the intense heat emitted by it. It also discusses the limitations imposed by the harsh realities of the solar environment on human physiology, providing a more comprehensive and accurate response to the question.\\
					\bottomrule
			\end{tabular}}
			 \label{tbl:exampleanswer}
		\end{table}

%
%
\bibliographystyle{splncs04}
\bibliography{mybibliography}
%

\end{document}